\documentclass{article}
\usepackage{fullpage}

\usepackage{amsmath, amssymb,bm}
\usepackage{hyperref}
\usepackage{mathtools}
\usepackage{amsthm}
\usepackage{cite}
\usepackage{microtype}
\usepackage{enumitem}
\usepackage{mathrsfs}
\usepackage{xcolor}
\usepackage{standalone}
\usepackage{tikz}
\usepackage{subcaption}
\usepackage{pgfplots}

\usetikzlibrary{positioning,decorations.pathreplacing,shapes}
\usetikzlibrary{arrows.meta}

\theoremstyle{definition}
\newtheorem{theorem}{Theorem}
\newtheorem{cor}[theorem]{Corollary}
\newtheorem{lemma}[theorem]{Lemma}

\newtheorem{definition}{Definition}
\newtheorem{remark}{Remark}

\newtheorem{condition}{Condition}

\newtheorem{example}{Example}

\definecolor{LineBlue}{rgb}{.2 .6 .9}
\definecolor{LineRed}{rgb}{.9 .3 .3}
\definecolor{LineGreen}{rgb}{.3 .8 .3}
\definecolor{LinePurple}{rgb}{.8 .3 .8}

\newcommand{\cB}{\mathcal{B}}

\newcommand{\cL}{\mathcal{L}}

\newcommand{\cU}{\mathcal{U}}

\newcommand{\cW}{\mathcal{W}}
\newcommand{\cX}{\mathcal{X}}

\newcommand{\bbR}{\mathbb{R}}
\newcommand{\bbN}{\mathbb{N}}

\newcommand{\sD}{\mathsf{D}}

\newcommand{\one}{\boldsymbol{1}} 
 
\newcommand{\Id}{\mathrm{I}}
\newcommand{\Comm}{\mathrm{K}}

\newcommand{\E}{\ensuremath{\mathbb{E}}}
\newcommand{\bPr}{\ensuremath{\mathbb{P}}}
\DeclarePairedDelimiter\bkt{[}{]}     
\makeatletter
\newcommand{\@exstar}[1]{\E \bkt*{#1}}
\newcommand{\@exnostar}[2][]{\E \bkt[#1]{#2}}
\newcommand{\ex}{\@ifstar\@exstar\@exnostar}
\makeatother

\makeatletter
\newcommand{\@prstar}[1]{\bPr \bkt*{#1}}
\newcommand{\@prnostar}[2][]{\bPr \bkt[#1]{#2}}
\newcommand{\pr}{\@ifstar\@prstar\@prnostar}
\makeatother

\DeclareMathOperator{\cov}{\mathsf{Cov}}

\DeclareMathOperator{\var}{\sf Var}

\DeclareMathOperator{\Lip}{Lip}

\newcommand{\dd}{\mathrm{d}}

\DeclareMathOperator{\gtr}{tr}
\DeclareMathOperator{\gvec}{\mathsf{vec}}

\DeclarePairedDelimiter{\norm}{\lVert}{\rVert}

\makeatletter
\newcommand{\@normopstar}[1]{\norm*{#1}_\mathrm{op}}
\newcommand{\@normopnostar}[2][]{\norm[#1]{#2}_{\mathrm{op}}}
\newcommand{\normop}{\@ifstar\@normopstar\@normopnostar}
\makeatother

\newcommand{\normal}{\mathsf{N}}

\pgfplotsset{compat=newest} 

\let\originalleft\left
\let\originalright\right
\renewcommand{\left}{\mathopen{}\mathclose\bgroup\originalleft}
\renewcommand{\right}{\aftergroup\egroup\originalright}

\title{Dimension-Free  Bounds for Generalized First-Order Methods\\ via Gaussian Coupling}

\author{Galen Reeves}
 
\begin{document}

\maketitle

\begin{abstract}
We establish non-asymptotic bounds on the finite-sample behavior of generalized first-order iterative algorithms --- including  gradient-based optimization methods and approximate message passing (AMP) --- with Gaussian data matrices and full-memory, non-separable nonlinearities. The central result constructs an explicit coupling between the iterates of a generalized first-order method and a conditionally Gaussian process whose covariance evolves deterministically via a finite-dimensional state evolution recursion. This coupling yields tight, dimension-free bounds under mild Lipschitz and moment-matching conditions. Our analysis departs from classical inductive AMP proofs by employing a direct comparison between the generalized first-order method and the conditionally Gaussian comparison process.  This approach provides a unified derivation of AMP theory for Gaussian matrices without relying on separability or asymptotics. A complementary lower bound on the Wasserstein distance demonstrates the sharpness of our upper bounds. 
\end{abstract}

\setcounter{tocdepth}{2} 
\tableofcontents

\section{Introduction }
Iterative algorithms play a central role in modern high-dimensional statistics and optimization.  A broad class of such algorithms  --- including gradient-based optimization methods and approximate message passing (AMP) --- can be represented (after symmetrization) as a generalized first-order method of the form:
\begin{align}
x_t = A  f_t(x_1, \dots, x_{t-1}  )   + g_t(x_1, \dots, x_{t-1}).   \label{eq:xt}
\end{align}
Here, each $x_t \in \bbR^n$ is an iterate in a discrete-time sequence, $A \in \bbR^{n \times n}$ is a symmetric matrix,  and $f_t, g_t \colon (\bbR^n)^{t-1} \to \bbR^n$ are  functions that depend on the previous iterates. This recursion  encompasses many widely used algorithms and allows for non-separable, full-memory nonlinearities.  

The statistical behavior  in the presence of random data has attracted significant attention. A major contribution of the AMP framework~\cite{donoho:2009a,bayati:2011}  has been to provide an explicit connection between the  dynamics  of \eqref{eq:xt} and  a limiting Gaussian process whose mean and variance evolve deterministically according to the state evolution formalism. This connection has been instrumental for the design and analysis of algorithms for applications in regression \cite{donoho:2009a,bayati:2011,rangan:2011,bayati:2012,javanmard:2013a,rangan:2019,tan:2024pooled}, coding\cite{rush:2017aa},  matrix factorization\cite{fletcher:2018,parker:2014b, deshpande:2014, lesieur:2017, Montanari:2021a,behne:2022,barbier:2023,rossetti:2023_amp_mtp,rossetti:2024isit,rossetti:2024_opamp}, and sampling \cite{el-alaoui:2022,montanari:2024posterior,huang:2024sampling,cui:2024_sampling}.  It also served as a powerful proof technique for understanding the structural properties of optimization problems---independent of the specific algorithm used. More recently, these insights have been extended to  broader classes of generalized first-order methods \cite{celentano:2020estimation} and  continuous-time analogs \cite{celentano:2021highdimensional,gerbelot:2024rigorous,montanari:2024statistically,fan:2025dynamical}, enabling rigorous formulations of dynamical mean-field theory. 

However, existing theoretical tools often impose restrictive assumptions such as separability or asymptotic limits. In this work, we depart from such constraints and present a new, non-asymptotic coupling framework that applies broadly to generalized first-order methods with Gaussian data matrices. 
This is achieved by comparing the behavior of such systems over a fixed number of time steps $T \ll n$  to a conditionally Gaussian process of the form:
\begin{align}
y_t = m_t(y_1, \dots, y_{t-1} ) + w_t \label{eq:yt} 
\end{align}
where $w_t \in \bbR^n$ is a mean-zero Gaussian process with covariance $\cov(w_s, w_t) = \Sigma_{st} \Id_n$ and the functions $m_t \colon  (\bbR^n)^{t-1} \to \bbR^n$ are constructed to match moments with the original system. In contrast to the original process, the conditional covariance of $y_t$ given its history is independent of past states, allowing for simplified analysis.

We analyze the behavior of system \eqref{eq:xt} when applied to Gaussian random matrices. For concreteness, our results are stated with respect to the  Gaussian orthogonal ensemble $\mathsf{GOE}(n)$ where  $A$ is a symmetric matrix with independent $\normal(0,2/n)$ entries on the diagonal and independent $\normal(0,1/n)$ entries below the diagonal. Following the arguments in  \cite{rossetti:2024isit,rossetti:2024_opamp}, results established for the GOE extend naturally to a broader class of matrix-variate Gaussian distributions.   In particular, all of our results can be applied to recursions defined on asymmetric Gaussian matrices with correlated rows and columns.

Our main results provide a constructive coupling between the original sequence $(x_t)$ and the approximation  $(y_t	)$, together with dimension-free bounds on the coupling error. Under regularity assumptions on the functions $f_t$ and $g_t$ and a suitable moment-matching condition, we show that the total error between the two processes concentrates at a scale that is independent of the ambient dimension.  A special case of our main results is summarized in the following theorem: 

\begin{theorem}[Dimension-Free Upper Bound]\label{thm:intro} 
Assume that $A\sim \mathsf{GOE}(n)$ and the functions  $f_t$ and $g_t$ are $L$-Lipschitz continuous for all $t \in [T]$. Let the parameters $(m_t)_{t \in T}$ and $\Sigma = (\Sigma_{st})_{s,t \in T}$ be matched to \eqref{eq:xt} according to the recursive  construction in  Definition~\ref{def:SE} and assume that $\Sigma$ is positive definite. Then, there exists a coupling of \eqref{eq:xt} and \eqref{eq:yt} such that
\begin{align*}
\pr[\Big]{ \max_{t \in [T]} \norm{ x_t - y_t} \ge C   \sqrt{r}} \le 2 \, e^{ -  r} , \qquad \text{for all $0 \le r \le n$,}
\end{align*}
where $C$ is a positive constant that depends only on the number of time steps $T$, the Lipschitz constant $L$, and the largest and smallest eigenvalues of the covariance matrix $\Sigma$.
\end{theorem}

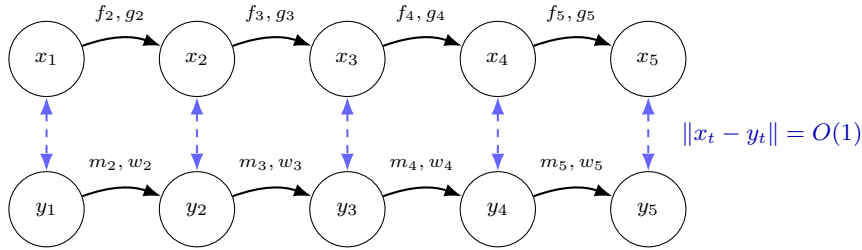
\begin{figure}
\usetikzlibrary{arrows.meta, positioning, calc, backgrounds, fit}
\centering

\begin{tikzpicture}[
    node distance=1.5cm and 2.2cm,
    every node/.style={font=\footnotesize},
    iter/.style={circle, draw, minimum size=1cm},
       coupling/.style={dashed, <->, thick, blue!60, arrows={Latex-Latex}},
    dependency/.style={-{Latex}, thick},
    label/.style={font=\scriptsize}
]

\foreach \i in {1,2,3,4,5}
{
    \node[iter] (x\i) at (\i*2, 2) {$x_{\i}$};
    \node[iter] (y\i) at (\i*2, 0) {$y_{\i}$};
    
    \draw[coupling] (x\i) -- (y\i);
}

\foreach \i/\j in {1/2, 2/3, 3/4, 4/5}
{
    \draw[dependency, bend left=20] (x\i) to (x\j);
}

\foreach \i/\j in {1/2, 2/3, 3/4, 4/5}
{
    \draw[dependency, bend left=20] (y\i) to (y\j);
}

\foreach \i in {2,...,5} {
    \node[label, above left =0.02cm and 0.2cm of x\i] (fx\i) {$f_{\i}, g_{\i}$};
}


\foreach \i in {2,...,5} {
    \node[label, above left =0.03cm and 0.1 cm of y\i] (my\i) {$m_{\i}, w_{\i}$};
}


\node[align=center, font=\scriptsize, blue!70!black] at (11.7, 1.0) {\small $\|x_t - y_t\| = O(1)$ };

\end{tikzpicture}

\caption{
An explicit coupling between the generalized first-order method $x_t$ and a conditionally Gaussian comparison process $y_t$ enables sharp, non-asymptotic analysis of iterative algorithms with Gaussian data matrices. Owing to its simplified structure,  the dynamics of the comparison process be described accurately via the state evolution formalism. The coupling ensures that these dynamics can be faithfully transferred back to the original system, providing dimension-free performance guarantees. }
\end{figure}

Our results advance existing theory by providing a sharp characterization of the dependence on dimension, under extremely general assumptions on the functions.  Under similar structural assumptions, the work in  \cite{berthier:2020,gerbelot:2023,montanari:2024statistically} shows that the difference between the processes is  $o(\sqrt{n})$, which is equivalent to stating that the normalized error $\frac{1}{\sqrt{n}} \norm{ x_t - y_t}$  vanishes asymptotically—but no explicit rate of convergence is given. Meanwhile, nearly all prior work that does provide explicit convergence rates \cite{rush:2018, cademartori:2023, li-fan-wei:2023, li-wei:2023,han:2025_dynamics}  imposes strong structural constraints—most notably, that the functions are row-separable, meaning the $i$-th output of $f_t$ depends only on the $i$-th row of the input. Among other things, these assumptions preclude the analysis of matrices with correlated entries, limiting applicability in more realistic settings.

The dependence on the number of time steps $T$ is analyzed in Section~\ref{sec:upper_bounds}. Under standard regularity conditions, $T$ can grow with the dimension at rate  $T = o( \log n)$. This improves upon the previous best known bound   $o( \log n / \log \log n)$, and is suitable for a broad class of applications. Moreover, as demonstrated in Section~\ref{sec:lower_bound}, this scaling is tight in certain settings,  meaning it cannot be improved without introducing additional assumptions on the problem.

\paragraph{Outline of contributions:} 
\begin{enumerate}

\item \textbf{Constructive Gaussian Coupling:}  We provide a constructive coupling between the two processes.  
The key idea is to  introduce an independent copy of the matrix $A$ that provides a stochastic correction term for \eqref{eq:xt} that exactly matches the target distribution defined by \eqref{eq:yt}.  

\item \textbf{Non-Asymptotic Bounds:}  We provide explicit bounds on the coupling error. 
Our approach applies to  arbitrary choices of parameters $(m, \Sigma)$ and the coupling error is controlled in terms of their mismatch with the system parameters $(f, g)$.  Assuming Lipschitz continuity, we use Gaussian concentration of measure to obtain dimension-free guarantees without requiring separability and independence assumptions. 

\item  \textbf{Tightness via Wasserstein Lower Bounds:}   We establish matching lower bounds on the quadratic Wasserstein distance  (Theorem~\ref{thm:couplingLB}) demonstrating that our upper bounds cannot be improved without stronger assumptions.  
\end{enumerate}

From a pedagogical point of view, a further contribution of our approach is that it provides an entirely self-contained and relatively simple derivation of the AMP framework for Gaussian matrices. In particular, our results can be applied generally for any choice of parameters $(m, \Sigma)$. Optimizing over these parameters then recovers the usual specifications given in the literature.

\paragraph{Notation} For a vector $v \in \bbR^d$ we use $\norm{v}$ to denote the Euclidean norm. 
For a matrix $M \in \bbR^{ m \times n}$ we use $\norm{M}$ to denote the Frobenius norm  and $\normop{M}$ to denote the Euclidean operator norm i.e.,  the spectral norm. For a sequence $(x_t)_{t \in \bbN}$ of elements in space $\cX$  we use $x_{\le t}= \gvec(x_1, \dots, x_t)$ to denote the vector obtained by stacking  to the first $t$ terms.  All functions are assumed to be measurable. A function  $f \colon \cX \to \bbR^n$  with $\cX \subseteq \bbR^n$ is Lipschitz continuous with Lipschitz constant $L \ge 0 $ if  $\norm{f(x) - f(y)} \le L \norm{x- y}$ for all $x, y \in \cX$. The smallest Lipschitz constant is denoted by  $\Lip(f) \coloneqq \inf\{ L \ge 0  \mid \text{ $f$ is $L$-Lipschitz}\}$.

\section{Background}

\subsection{Approximate Message Passing} 

AMP is a powerful framework for designing and analyzing iterative algorithms in high-dimensional inference problems. 
It was first proposed by Donoho, Maleki, and Montanari~\cite{donoho:2009a} in the context of sparse linear regression and its asymptotic behavior was proved rigorously by  Montanari and Bayati~\cite{bayati:2011}. Ensuing work has strengthened and generalized these results by  relaxing the assumptions on the functions \cite{rangan:2011,javanmard:2013a,berthier:2020,gerbelot:2023}, allowing for a broader class of matrices \cite{opper:2016, ma:2017ab, rangan:2019, fan:2022,dudeja:2023,zhong:2021}, and allowing for partial updates~\cite{rossetti:2024_opamp}. For a comprehensive overview, see~\cite{feng:2022}.

While AMP algorithms may vary considerably across applications, many of them can be unified under a common structural framework.  After symmetrization, the basic form of an  AMP algorithm can be modeled as a special case of the system \eqref{eq:xt} given by
\begin{align}
x_t = A  f_t(x_{<t}  )   - \sum_{s < t} b_{st}   f_s(x_{<s}) 
\label{eq:amp} 
\end{align}
where  $(b_{st} )_{1 \le s < t}$ are scalar debiasing coefficients designed to decorrelate the iterates. The term  $ \sum_{s < t} b_{st}   f_s$ is referred to as the Onsager correction term \cite{donoho:2009a}.

\paragraph{Debiasing Coefficients and State Evolution.} State evolution defines a Gaussian  process whose mean and covariance are updated recursively in terms of the functions used in the AMP system.  While the original formulation of state evolution was tied to row-separable functions \cite{bayati:2011, javanmard:2013a}, this was extended to non-separable pseudo-Lipschitz functions by Berthier et al.~\cite{berthier:2020}.  The state evolution in \cite{berthier:2020} defines a zero-mean Gaussian process $(w_t)_{t \in \bbN}$ with covariance of the form $\cov(w_s, w_t) = \Sigma_{st} \Id_n$ for scalars $(\Sigma_{st})_{s,t \in \bbN}$. The  covariance is initialized with $\Sigma_1 = \frac{1}{n} \|f_1\|^2$, where  $f_1$ is a constant function that defined the initialization of the algorithm. The  covariance parameters  for time $t$ are then defined recursively in terms of $\Sigma_{<t} = (\Sigma_{rs})_{1 \le r,s < t}$ according to 
\begin{align}
\Sigma_{st} & = \frac{1}{n} \ex{ \langle f_s (w_{< s} ), f_t(w_{<t}) \rangle},
\end{align}
 where the expectation is with respect to $w_{<t} \sim \normal( 0, \Sigma_{<t} \otimes \Id_n)$ 
 
 The link between the AMP recursion in  \eqref{eq:amp} and the state evolution process hinges critically on the precise specification of the debiasing coefficients. In practice, the debiasing terms are carefully designed functions of the previous states.  For the purposes of theoretical analysis, it is convenient to consider deterministic coefficients that are defined in terms of the evolution process. Assuming  the $f_t$ are weakly differentiable, these coefficients are given by \cite{berthier:2020}
 \begin{align}
b_{st} & =  \frac{1}{n} \ex*{ \gtr\left( \frac{ \dd f_t}{ \dd w_s} \right)}, \label{eq:bst_amp} 
\end{align}
where  $\frac{ \dd f_t}{ \dd w_s}$ denotes the $n \times n$ Jacobian matrix of partial derivatives of $f_t$ evaluated at $y_{<t}$.  
  
\paragraph{Impact of AMP.}
The AMP framework has been applied successfully for high-dimensional inference problems, including regression~ \cite{donoho:2009a,bayati:2011,rangan:2011,bayati:2012,javanmard:2013a,rangan:2019,tan:2024pooled}, coding\cite{rush:2017aa},  matrix factorization\cite{fletcher:2018,parker:2014b, deshpande:2014, lesieur:2017, Montanari:2021a,behne:2022,barbier:2023,rossetti:2023_amp_mtp,rossetti:2024isit,rossetti:2024_opamp}, and sampling \cite{el-alaoui:2022,montanari:2024posterior,huang:2024sampling,cui:2024_sampling}. 
The framework also provides asymptotically exact formulas for performance matrices, such as mean-squared error, that can be  compared against information-theoretic limits~\cite{donoho:2013,deshpande:2017,reeves:2019c,barbier:2019,reeves:2020} to identify regimes where iterative   methods are provably optimal.

\subsection{Generalized First-Order Methods}

Generalized first-order methods  \cite{celentano:2020estimation}  extend the classical AMP framework by allowing more general update rules of the form in \eqref{eq:xt}.  This formulation includes algorithms such as gradient descent, momentum methods, and memory-augmented AMP variants as special cases. Recent work \cite{celentano:2021highdimensional,gerbelot:2024rigorous,montanari:2024statistically,fan:2025dynamical} has used this generalization to provide rigorous formulations of dynamical mean-field theory. 

Formally, a generalized first-order method of the form given in  \eqref{eq:xt} and the original AMP formulation in \eqref{eq:amp} are functionally equivalent in the sense that one can be transformed to the other via a change of variables~\cite{montanari:2024statistically}. However, this mapping is implicit and it can be difficult to see how the corresponding transformation carries over to the state evolution analysis. 

\paragraph{State Evolution.} In the context of generalized first-order methods described by  \eqref{eq:xt},  the objective of state evolution is to establish an explicit mapping from the  function sequences $(f, g)$ to the parameters $(m, \Sigma)$ that characterize the comparison process in \eqref{eq:yt}. This mapping was described in \cite{celentano:2020estimation} for specific applications, and more recently in \cite{han:2025_dynamics} for row-separable functions. However, to our knowledge, there are no general results that apply to non-separable functions without relying on implicit changes of variables that reduce the analysis to the AMP setting.

In the definition below, we propose an explicit recursive construction that maps $(f, g)$ directly to the parameters $(m , \Sigma)$.  This construction involves scalar coefficients  $(b_{st})$, which are closely related to the debiasing coefficients appearing in the AMP framework.

\begin{definition}[State Evolution]  \label{def:SE} 
The mean and covariance  are initialized with $m_1 = g_1$ and $\Sigma_1 = \frac{1}{n} \norm{f_1}^2$. Then, assuming that $f_t$ and $g_t$ are square integrable functions of $y_{<t}$, the parameters for time $t$ are defined recursively according to  
\begin{subequations} 
\begin{align}
m_t(\cdot ) & = g_t(\cdot ) + \sum_{s < t}  b_{st}  \, f_s(\cdot )\\
b_{st} & = \frac{1}{n} \sum_{r \le s}\big [\Sigma^+_{\le t}\big]_{rs}
\, \ex{ \langle y_r - m_r(y_{<r}) ,  f_t(y_{<t})  \rangle }, \label{eq:bts}\\
\Sigma_{st} & = \frac{1}{n} \ex{ \langle f_s (y_{< s} ), f_t(y_{<t}) \rangle} , 
\end{align}
\end{subequations}
where  $\Sigma_{\le t}\coloneqq (\Sigma_{rs})_{1 \le r,s \le t}$ is the leading principal submatrix of $\Sigma$  and $(\cdot)^+$ denotes the  Moore-Penrose pseudoinverse of a matrix. 
\end{definition} 

In contrast to the  standard characterizations of state evolution, this formulation requires only the existence of finite second moments at each iteration and the functions are not required to be differentiable. Assuming differentiability,  Stein's Lemma (a.k.a.\ Gaussian integration by parts) can be applied to provide an alternative expression for the coefficients in \eqref{eq:bts} that generalizes gradient-based characterization for AMP given in \cite{berthier:2020}.

\begin{lemma}\label{lem:SE}
Suppose that the function sequences $(f_t)$ and $(g_t)$ are weakly differentiable and their derivatives are square integrable functions of $y_{\le t}$. Then,  \eqref{eq:bts} can be expressed as
\begin{align}
b_{st} & = \frac{1}{n} \ex*{ \gtr\left( \frac{ \dd f_t}{ \dd y_s} \right) }  +  \frac{1}{n} \ex*{ \gtr\left( \frac{ \dd f_t}{ \dd y_{\le t}}  \left( \Id_{nt} -  \frac{ \dd m_{\le t}}{ \dd y_{\le t}}  \right)^{-1}  \frac{ \dd m_{\le t}}{ \dd y_s}  \right) },  \label{eq:bts_stein}
\end{align}
where $\frac{\dd f_t}{ \dd y_{\le t}} \in \bbR^{n \times nt}$ and  $\frac{\dd  m_{\le t}}{ \dd y_{\le t}} \in \bbR^{nt \times nt}$ denote the Jacobian matrices of $f_t$ and $m_{\le t}$ respectively. 
\end{lemma} 

\begin{remark} 
The fact that each $m_s$ depends only on inputs before time $s$ implies that $\frac{ \dd m_s}{ \dd y_r} = 0_{n \times n} $ for all $r \ge s$, and thus  the Jacobian matrix of $m_{\le t}$ is a strictly lower triangular block matrix.  This ensures that the right-hand side of \eqref{eq:bts_stein}  is well-defined as a function of the distribution of $y_{<t}$ and the coefficients $(b_{rt})_{1 \le r < s}$. 
\end{remark}

\begin{remark} 
The AMP framework can be recovered from this general construction by specifying each  $g_t$ as a function of $(f_s)_{s < t}$ such that $m_t \equiv 0$. In such cases, we see that \eqref{eq:bts_stein} reduces to the specification of the debiasing coefficients in \eqref{eq:bst_amp}.
\end{remark}

\subsection{Connection with Related Work} 

There has been growing interest in non-asymptotic guarantees for AMP-type methods, particularly those that aim to characterize finite-sample behavior. Notable results include exponential convergence rates for empirical measures\cite{rush:2018, cademartori:2023}, explicit distributional decompositions \cite{lu:2021dice, li-fan-wei:2023, li-wei:2023}, refined analysis  in the context of the Sherrington-Kirkpatrick  model~\cite{celentano:2023,celentano:2023a}, and recent progress on entrywise guarantees for generalized first-order methods~\cite{han:2025_dynamics}.  However, most of these results rely on structural assumptions such as separable nonlinearities, which are more restrictive than those used in the asymptotic AMP theory \cite{berthier:2020}.

In contrast, this paper develops a direct coupling framework that compares the iterates of a generalized first-order method with a conditionally Gaussian process via a one-shot argument, rather than relying on the inductive techniques common in prior AMP analyses. This coupling yields tight, non-asymptotic bounds on the deviation between the two processes over a fixed number of iterations, under general full-memory and non-separable settings.

Some specific connections are as follows:  

\begin{itemize} 
\item Rush and Venkataramanan~\cite{rush:2018}  study AMP in the regression setting under separable nonlinearities, showing exponential convergence of the empirical measure to a limiting process defined via state evolution. 
 Their analysis allows the number of time steps to scale sub-logarithmically with the dimension according to $T  = o( \frac{  \log n}{ \log \log n}$). 

\item Berthier et al.~\cite{berthier:2020} and Gerbelot and Berthier \cite{gerbelot:2023} focus on the AMP recursion in \eqref{eq:amp} with non-separable functions and  establish convergence with respect to sequences of pseudo-Lipschitz test functions. As shown in \cite[Appendix~C]{rossetti:2024_opamp}, their results imply and are implied by the existence of a sequence (in $n$)  of couplings under which $n^{-1/2}\norm{x_{\le t} - y_{\le t}} \to 0$ in probability as $n \to \infty$. Under the same regularity conditions, the results in this paper show that the typical deviations of $\norm{x_{\le t} - y_{\le t}} $ can be bounded independently of the problem dimension. 

\item  Lu~\cite{lu:2021dice} provides a distributional decomposition of the AMP algorithm for the purposes of enabling rapid simulation of AMP sample paths without generating a full $n \times n$ matrix.  While our coupling does not follow this route, such decompositions may offer alternative paths to sharp non-asymptotic analysis.

\item Aloui et al.~\cite{el-alaoui:2022} and Celentano et al.~\cite{celentano:2023} focus on a symmetric rank-one matrix model (i.e., the Sherrington-Kirkpatrick model)  and show that AMP can be viewed as an optimization technique to find the critical points of the TAP free energy. By leveraging geometric properties of the local landscape of the objective function near critical points, they provide  precise theoretical guarantees in regimes where the number of iterations scales polynomially with the dimension.  Our theory differs in scope in the sense that it is designed to be applied under the general setting of the asymptotic theory \cite{berthier:2020}. 

\item A different approach to the non-asymptotic analysis is proposed in the recent work of  Li~et~al. \cite{li-fan-wei:2023} and Li and Wei~\cite{li-wei:2023}, who consider a decomposition of the AMP process into an orthogonal basis defined by the sample path. This high level idea of this decomposition is to distinguish between dominant terms, which are approximately Gaussian,  and certain residual terms, which are presumed to be negligible. Focusing on special cases of spiked matrix model with separable functions (similar to \cite{celentano:2023}), a detailed analysis of the residual terms yields rigorous guarantees for AMP algorithms where the number of iterations scales polynomially with the dimension.

\item Finally, the recent work of Han~\cite{han:2025_dynamics} focuses directly on generalized first-order methods of the form in \eqref{eq:xt} and establishes non-asymptotic entrywise guarantees under the assumptions of row-separable nonlinearities with bounded first, second, and third derivatives. In particular,  \cite[Theorem~2.4]{han:2025_dynamics} provides an upper bound on the distributional distance between the $i$-th row of $x_{\le t}$ and the corresponding row of the comparison process  $y_{\le t}$ for each $i \in [n]$. For a suitably smooth class of test functions, the rate of convergence is  $ O( (\log n)^{c_1 T^5} n^{-1/c_0^T} )$ for positive constants  $c_0, c_1$. In addition, Han proves a universality result showing that the same results hold when the matrix $A$ is drawn from a symmetric ensemble with independent entries and matched first and second moments. 
Compared to our results, which apply to general nonlinearities under Gaussian assumptions, \cite{han:2025_dynamics} trades off generality for universality, addressing a broader class of matrix ensembles at the cost of stronger smoothness assumptions.
\end{itemize}

\section{Main Results} 

We establish theoretical guarantees on the closeness between the distributions of the generalized first-order method in  \eqref{eq:xt} and the comparison process in  \eqref{eq:yt} by showing that there is coupling, i.e., a joint distribution on the two processes,  under which their difference is bounded independently of the dimension with high probability.  The significance of these results is that the behavior of generalized first-order methods can be analyzed precisely by studying the much simpler comparison process.

For given number of  time steps  $T \in [n]$, we represent the sequences  generated by \eqref{eq:xt} and \eqref{eq:yt}  in terms of the $n \times T$  matrices 
\begin{align*}
X = [x_1, \dots, x_T] , \qquad Y = [ y_1, \dots, y_T] .
\end{align*}
A coupling of these processes is any specification of a joint distribution on the pair  $(X, Y)$ whose marginal distributions satisfy \eqref{eq:xt} and \eqref{eq:yt}. 

We also define the matrix-valued functions: 
\begin{align*}
F = [f_1, \dots, f_T] , \qquad G  = [g_1, \dots, g_T] ,  \qquad M = [m_1 , \dots, m_T].
\end{align*}
We use the notation $F(X)$  to express the evaluation of function $F$ on the matrix $X$.

\subsection{Coupling Construction} \label{sec:coupling}

The following construction defines a joint distribution on the matrix $A$ and the process  $(y_t)_{t \in [T]}$. It applies generally to  any collection of parameters $(f, g, m, \Sigma)$.

\begin{definition}[Coupling Construction]\label{def:coupling}   Let $A, A' \overset{\mathrm{iid}}{ \sim }\mathsf{GOE}(n)$, and let $\Omega$ be a $T \times T$ upper triangular matrix with nonnegative diagonal entries such that  $\Sigma = \Omega^\top  \Omega$ is the Cholesky decomposition of $\Sigma$. 
Let $(v_1, \dots, v_n)$ be a fixed ordered basis for $\bbR^n$.   For each $t = 1, \dots, T$, perform the following steps:
\begin{enumerate}
\item If $f_t(y_{<t})$ does not lie in the span of $\{q_1, \dots, q_{t-1}\}$ then set $q_t \in\bbR^n$ to be the vector of length $\sqrt{n}$ such that 
\begin{align}
q_t \propto  f_t(y_{<t}) - \sum_{s < t} \tfrac{1}{n}  \langle q_s , f_t(y_{<t})  \rangle q_s \label{eq:coupling_qt}
\end{align}
Otherwise, substitute $f_t(y_{<t})$ with the first basis vector in $(v_1, \dots, v_n)$ that is not in the span of $\{q_1, \dots, q_n\}$. 

\item Update the state variables $(z_t, w_t, y_t)$  according to 
\begin{subequations}
\label{eq:coupling_update}
\begin{align}
z_t & = A q_t  + \frac{1}{2} ( A'_{tt} -  \tfrac{1}{n} \langle q_t, A q_t \rangle  ) \,  q_t +   \sum_{s  < t} ( A'_{st} -  \tfrac{1}{n} \langle q_s, A q_t \rangle  )  \, q_s 
\label{eq:coupling_zt}\\
w_t & = \sum_{s  \le  t} \Omega_{st} z_{s} \label{eq:coupling_wt} \\
y_t  &= m_t(y_{<t})  +w_t \label{eq:coupling_yt}
\end{align}
\end{subequations}
\end{enumerate}
\end{definition}

\begin{remark}
The ordered basis $(v_1, \dots, v_n)$  ensures uniqueness of the construction when $f_t$ depends linearly on previous evaluations. In principle, any basis suffices  and one may use the standard basis or a randomly generated basis. 
\end{remark}

The following result shows that the construction in Definition~\ref{def:coupling} produces a sequence $(y_t)_{t \in [T]}$ with the desired marginal distribution specified by \eqref{eq:yt}. If  $(x_t)_{t \in[T]}$ is generated from \eqref{eq:xt} using the same matrix $A$, then this defines a coupling between the two processes.
  
\begin{theorem}\label{thm:coupling} 
The sequence $(z_t)_{t \in [T]}$ constructed in  Definition~\ref{def:coupling} has independent standard Gaussian elements $z_t \sim \normal(0, \Id_n)$. Consequently,  $(w_t)_{t \in [T]}$ is a zero-mean Gaussian process with covariance $\cov(w_s, w_t) = \Sigma_{st} \Id_n$, and   $(y_t)_{t \in [T]}$ follows the distribution in \eqref{eq:yt}. Moreover,  the following identity holds for all  $1 \le t \le r \le T$:
\begin{align}
z_t = A q_t  -  \frac{1}{n}   \sum_{s  \le r } \langle z_s, q_t \rangle \,    q_s  +   \sum_{s \le r }  A'_{st} q_s.
\label{eq:ztqt_coupling}
\end{align}. 
\end{theorem}

\subsection{Non-Asymptotic Upper Bounds}\label{sec:upper_bounds}

We now use the coupling in Definition~\ref{def:coupling} to provide an explicit bound on the coupling error. We will  assume throughout this section that $\Sigma$ is positive definite with (necessarily unique) Cholesky decomposition $\Sigma = \Omega^\top \Omega$. For the next result, we  define  the $T \times T$ random matrices 
\begin{align*} 
 B_n  \coloneqq \tfrac{1}{n} \Sigma^{-1}  (Y - M(Y))^\top F(Y), \qquad  S_n  & \coloneqq  \tfrac{1}{n}  \Omega^{-\top} F(Y)^\top F(Y)  \Omega^{-1} ,
\end{align*}
which can be viewed as the sample covariance matrices defined by  the  rows of  $F(Y) \Omega^{-1}$ and $(Y- M(Y)) \Sigma^{-1}$. 

Additionally,  we use $\norm{\, \cdot\, }_{2,1}$ to denote the $\ell_2 \to \ell_1$ operator norm, i.e., the sum of the column norms. For an $n \times k$ matrix $V = [v_1, \dots, v_k]$, this norm satisfies $\norm{V}_{2,1} \coloneqq \sum_{j=1}^k \norm{v_j}  \le \sqrt{k} \,  \norm{V} $.  

\begin{theorem}\label{thm:Delta12} 
Assume that $\Sigma$ is positive definite with Cholesky decomposition $\Sigma = \Omega^\top \Omega$ and  the functions $f_t$ and $g_t$ are Lipschitz with constants $L_f$ and  $L_g$ for all $t \in [T]$.  Then, under the coupling $(X, Y)$ from Definition~\ref{def:coupling}, the event
\begin{align*}
\norm{X- Y} \le    \big (1 + 4 L_f + L_g\big)^{T-1} \,  \big(  \Delta_1 + \Delta_2  +   2 \sqrt{T} +  \sqrt{2r}    \big)  \cdot  \norm{\Omega} _{2,1} 
\end{align*}
holds with probability at least $1 - 3 e^{-r}$ for all $0 \le r \le n$ where
\begin{subequations}
\label{eq:Deltas}
\begin{align}
\Delta_1 & \coloneqq     \norm{M(Y) - G(Y)-   F(Y)   B_n }_{2,1}  \norm{ \Omega}_{2,1}^{-1} \label{eq:Delta1} \\
\Delta_2 & \coloneqq  52 \log_2(4T)    \sqrt{n}      \normop{S_n  - \Id } . \label{eq:Delta2}
\end{align}
\end{subequations} 
\end{theorem}

Theorem~\ref{thm:Delta12}  shows that coupling error is controlled  by the random variables $(\Delta_1, \Delta_2)$.  The term $\Delta_1$ depends on the mismatch between the parameters $(f, g)$ and $(m, \Sigma)$, while  $\Delta_2$ quantifies the deviation between the empirical second moment matrix $\frac{1}{n} F(Y)^\top F(Y)$ and the target covariance $\Sigma$. For the purposes of analysis, an important property of this bound is that both variables are measurable with respect to the comparison process  $Y$. Thus, we can exploit structural properties of $Y$ to bound their distributions.

\begin{remark}
In principle, one could attempt to optimize the choice of $(m, \Sigma)$ as a function of $(f, g)$  to minimize the associated error terms. We note that the term $\Delta_1$ cannot be completely eliminated  because the entries $[B_{n}]_{st}$  for $s \ge t$ depend on the current and future states, and therefore cannot be used in the definition $m_t$, which must depend solely on past information.  
\end{remark}

\begin{remark}
The term $ (1 + 4 L_f + L_g)^{T-1}$ arises from an upper bound on the stability of system \eqref{eq:xt} given in Lemma~\ref{lem:stability2}. This part of the analysis can be improved at the cost of some additional complexity in the theorem statement and the assumptions. For example, the bound in  Lemma~\ref{lem:stability2} holds under the weaker assumption that there is an affine upper bound on the modulus of continuity of each $f_t$ and $g_t$ at the point $y_{\le t}$. 
\end{remark}

To provide explicit tail bounds on $(\Delta_1, \Delta_2)$, our strategy is to show that the random matrices $(B_n, S_n)$ concentrate around their expectations $(B, S)$.  Two basic scenarios where such concentration results can be established are as follows:
\begin{enumerate}
\item  \textbf{Independence:} If the functions $f_t$ and $m_t$  are row-separable --- meaning the $i$-th output depends only on the $i$-th input ---  then the rows of $Y$ and $F(Y)$ are independent. Assuming finite second moments, the multivariate central limit theorem implies  that  $(B_n, S_n)$ can be approximated by their expectations plus a random fluctuation term of order $1/\sqrt{n}$.  Under stronger moment assumptions, sharper convergence rates can be established using matrix concentration inequalities, such as matrix Bernstein's inequality~\cite{tropp:2015matrix_concentration}.

\item \textbf{Lipschitz Continuity:} If each function $f_t$ and $m_t$ is Lipschitz continuous, then the matrices $Y$ and $F(Y)$ can be expressed as Lipschitz continuous functions of  i.i.d.\ standard Gaussian variables. This structure permits the application of Gaussian concentration inequalities to rigorously control the deviations and establish convergence rates.
\end{enumerate} 

For our next result, we use the second approach to provide a refined upper bound the coupling error that does not require any separability conditions on the functions.  This result is stated with respect to the deterministic $T \times T$ matrices
\begin{align*}
B \coloneqq  \tfrac{1}{n} \Sigma^{-1}  \ex{ (Y - M(Y))^\top F(Y)} , \qquad S& \coloneqq  \tfrac{1}{n}  \Omega^{-\top}  \ex{ F(Y)^\top F(Y)  }  \Omega^{-1} .
\end{align*}

\begin{theorem}\label{thm:XY}
Assume that $\Sigma$ is positive definite with Cholesky decomposition $\Sigma = \Omega^\top \Omega$ and the functions $f_t$, $g_t$, $m_t$ are Lipschitz with constants $L_f, L_g, L_m$ for all $t \in [T]$.  
Then, under the coupling $(X, Y)$ from Definition~\ref{def:coupling}, the event
\begin{align*}
\norm{X- Y} & \le  c \cdot (1 + 4L_f + L_g)^{T-1}  \\
& \quad \times   \big(  \psi_1   +L \,   \psi_2  +   L^3 (\sqrt{T} + \sqrt{r})   \big)\,  \norm{\Omega}_{2,1}
\end{align*}
holds with probability at least $1 -5 e^{ - r}$ for all $0 \le r \le n$ where $c$ is a universal positive constant, 
\begin{align*}
L & \coloneqq \log_2(2T) + \sqrt{T} (L_f + L_g + L_m) ( 1 + L_m)^{T-1} \kappa(\Sigma)^{1/2}\\
\psi_1 &\coloneqq \ex[\big]{ \norm{  M(Y) - G(Y) - F(Y)  B}_{2,1}}  \cdot \norm{\Omega}_{2,1}^{-1} \\
\psi_2 &\coloneqq   \sqrt{n} \cdot \normop{ S  - \Id }  
\end{align*}
and $\kappa(\Sigma)$ is the condition number of $\Sigma$.
\end{theorem}

We emphasize that Theorem~\ref{thm:XY} holds generally for any choice of $(f, g, m, \Sigma)$ satisfying the Lipschitz continuity assumption and the positive definiteness assumption. If the parameters of the comparison process are matched to the original sequence following  the state evolution construction in Definition~\ref{def:SE}, then Theorem~\ref{thm:XY} holds with $\psi_1 = \psi_2 = 0$ and Theorem~\ref{thm:intro} follows immediately as a corollary.

\subsubsection{Dependence on Number of Steps.}  
Under the assumed scaling on the functions,  $\norm{x_t}$  and $\norm{y_t}$ are  typically of order $\sqrt{n}$. For fixed $T$ and  bounded $\psi_1, \psi_2$ our result shows that the fluctuations of the coupling error $\norm{X- Y}$ are order one, independent of the dimension. 

Our findings can also be applied in regimes where $T$ increases with $n$. The bound on the coupling error remains meaningful as long as $n^{-1/2} \norm{X- Y}$  vanishes asymptotically. Assuming the Lipschitz constants are bounded uniformly with respect to $(T,n)$ the upper bound on the typical deviations increases at rate $\exp\{ c_L \, T\} \kappa(\Sigma)^3$ where $c_L $ depends on the Lipschitz constants. A sufficient condition for $n^{-1/2} \norm{X- Y}$ to converge to zero in probability is then given by 
\begin{align}
T  +  \ln(\kappa(\Sigma))  = O(  \log n).  \label{eq:scaling} 
\end{align} 

It is important to recognize the role of the condition number in this setting. For many applications of interest, the state evolution converges to a fixed point, and the rate of convergence is exponential in $T$. On the one hand, the rapid convergence means that a fixed number of iterations is often sufficient to get arbitrarily close to the fixed point. On the other hand, closeness to a fixed-point necessarily implies that $\Sigma$ is near degenerate. To see this observe that from \eqref{eq:yt}, we have $\frac{1}{n} \ex{ \norm{y_{t+1} - y_t}^2} \ge  \frac{1}{n} \ex{ \norm{ w_{t+1} - \ex{ w_{t+1} \mid w_{\le t}}}^2}$. If this lower bound converges to zero then the smallest eigenvalue of $\Sigma$ also converges to zero, causing the condition number to grow without bound.  

To interpret our result in the context of the basic AMP recursion, we  evaluate with $g_t = -  \sum_{s < t}  b_{st} f_s$ and $m_t \equiv 0$ where $b_{st}$ are specified according to \eqref{eq:bts}.  Assuming each $f_t$ has Lipschitz constant $L_f$, we can apply the general upper bound
\begin{align}
\Lip(g_t) \le   (  \sum_{s < t} |b_{st} |^2 )^{1/2}  L_f \le \sqrt{T} L_f^2.  \label{eq:Lg_bound} 
\end{align}
Evaluating Theorem~\ref{thm:XY} with these constants, the relevant scaling becomes
\begin{align*}
T \ln(T)   +  \ln(\kappa(\Sigma))  = O(  \log n). 
\end{align*} 
Ignoring the dependence on the condition number, this result is consistent with the scaling condition $ T = o( \log n / ( \log \log n))$, which was previously established for special cases of AMP~\cite{rush:2018}. The extra $\ln(T)$ dependence on $T$ can be eliminated under stronger assumptions on the functions. For example, if each function has fixed memory that does not increase with $T$ (as is often assumed in the AMP literature)  then \eqref{eq:Lg_bound} implies that $L_g  = O(L_f^2)$, recovering the condition in \eqref{eq:scaling}.

\subsubsection{Extension to Pseudo-Lipschitz Functions.}  
Prior work on AMP  \cite{berthier:2020} has considered  functions satisfying a  pseudo-Lipschitz  condition of the form
\begin{align*}
\norm{f_t(u ) - f_t(v ) } \le  L \| u - v \| \Big ( 1 + \big( \tfrac{1}{\sqrt{n}} \|u \|\big)^{\alpha-1}  + \big( \tfrac{1}{\sqrt{n}} \|v \|\big)^{\alpha-1} \Big) \qquad \text{for all $u, v \in (\bbR^n)^{t-1}$,}
\end{align*} 
for positive constant $L$ and  exponent  $\alpha \in [1, \infty)$. The case  $\alpha = 1$ corresponds to the standard Lipschitz condition, while larger values $\alpha > 1$ allow for a broader class of functions. 

Theorem 5 extends naturally to this broader class of pseudo-Lipschitz functions. The key observation is that  for any $\rho > 0$, the restrictions of a pseudo-Lipschitz function $f$ (or $g$) to the set $\cB_n(\rho)  = \{ u \in (\bbR^n)^T \mid \| u\| \le \sqrt{n} \rho\}$  is Lipschitz continuous, with a Lipschitz constant depending only on $(L, \rho)$.  By Kirszbraun's extension theorem, such functions admit Lipschitz continuous extensions to all of $(\bbR^n)^T$ with the same Lipschitz constant. 

We can therefore apply the coupling construction from Definition~\ref{def:coupling} to these Lipschitz extensions, obtaining a coupled pair $(X, Y)$ whose error is bounded according to Theorem~\ref{thm:XY}. Thanks to the structural properties of $Y$ and the fact that $\norm{X-Y}$ is bounded with high probability, it follows that for sufficiently large $\rho$, the events $\gvec(X)  \in \cB_n(\rho)$ and $\gvec(Y)  \in \cB_n(\rho)$ both hold with probability at least  $1 - c_0 e^{ - c_1 n}$, for some constants $c_0, c_1$.  Since the Lipschitz extensions agree with the original pseudo-Lipschitz functions on $\cB_n(\rho)$, the coupling established for the extensions is valid for the original system as well.

\subsection{Lower Bound via Wasserstein Distance}\label{sec:lower_bound}

In this section, we establish the tightness of the upper bounds by proving a complementary lower bound on the quadratic Wasserstein distance. This result demonstrates that, in general, the coupling error cannot be improved without imposing additional structural constraints.

The quadratic Wasserstein distance between square integrable probability measures $\mu$ and $\nu$ on $\bbR^d$ is defined as
\begin{align*}
W_2(\mu, \nu) \coloneqq \left( \inf_{\pi \in \Pi(\mu, \nu)}  \int \| u -  v\|^2 \, \dd\pi(u, v)  \right)^{1/2},
\end{align*}
where $ \Pi(\nu, \mu)$  denotes the set of all couplings of $\mu$ and $\nu$.  Moreover, there exists an optimal coupling that attains this infimum~\cite[Theorem~4.1]{villani:2008}. 

For the random elements such as  $X$ and $Y$ we use $\cL(X)$ and $\cL(Y)$ to denote their distributions. The quadratic Wasserstein distance satisfies the sandwich relation
\begin{align}
 \bigg( \sum_{t =1}^T \cW^2_2( \cL(x_t), \cL(y_t))  \bigg)^{1/2}  \le \cW_2( \cL(X), \cL(Y)) \le \sum_{t =1}^T \cW_2( \cL(x_t), \cL(y_t))  \label{eq:W2_sandwich} 
\end{align}
The lower bound follows from relaxing the constraints on the coupling so that it only needs to match the column marginals. The upper bound follows from the triangle inequality. 

To obtain an exact expression for the quadratic Wasserstein distance, we consider the simplified setting in which each $f_t$ is a constant function (i.e., a fixed point in $\bbR^n$ that does not depend on its input) and the functions $g_t$ and $m_t$ are linear. 

\begin{condition}\label{cond:LB} For each $t \in T$, $f_t$ is a constant function (i.e., it does not depend on its input) and $g_t$ and $m_t$ are linear functions of the form
\begin{align*}
 g_t(x_{<t}) =    \sum_{s < t} \lambda_{st}\,  x_{t}, \qquad m_t(y_{<t}) =    \sum_{s < t} \gamma_{st} \,  y_{t}
\end{align*}
for strictly upper triangular matrices  $\Lambda = (\lambda_{st})$ and $\Gamma = (\gamma_{st})$. 
\end{condition}

Under the recursive construction of parameters  in Definition~\ref{def:SE}, the coefficients $b_{st}$ are zero because the $f_t$ are nonrandom, and the matrices $(\Gamma, \Sigma)$ are given by $\Gamma   =\Lambda$ and  $\Sigma= \tfrac{1}{n} F^\top F$ for all $s, t \in [T]$. Note that Condition~\ref{cond:LB}  implies that $f_t, g_t, m_t$ are Lipschitz continuous with constants 
\begin{align*}
L_f = 0, \qquad L_g  = \max_{t \in [T]} \Big(  \sum_{s< t} \lambda_{st}^2\Big)^{1/2}, \qquad L_m = \max_{t \in [T]} \Big( \sum_{s< t} \gamma_{st}^2\Big)^{1/2}.
\end{align*}

\begin{theorem}\label{thm:couplingLB}
Assume that Condition~\ref{cond:LB} holds. Then,  the quadratic Wasserstein distance satisfies
\begin{align*}
W_2^2\left( \cL( x_t)   , \cL(y_t) \right)  
& = \tfrac{ n-1}{n} ( \sqrt{2} - 1)^2  \alpha^2_t + n\,  \Big( \big( 1 + \tfrac{ \sqrt{2} - 1}{n} \big) \alpha_t - \beta_t \Big)^2,
\end{align*}
for all $ t \in [T]$ where
\begin{align*}
\alpha_t \coloneqq  \sqrt{   [ ( \Id - \Lambda)^{-\top} \tfrac{1}{n} F^\top F  ( \Id - \Lambda)^{-1} ]_{tt}  }  , \qquad \beta_t \coloneqq  \sqrt{   [ ( \Id - \Gamma)^{-\top} \Sigma ( \Id - \Gamma)^{-1} ]_{tt} } .
\end{align*}
\end{theorem}

From Theorem~\ref{thm:couplingLB}, we see that  the optimal specification of $(\Gamma, \Sigma)$ under the quadratic cost  is given by  $\Gamma = \Lambda$  and $ \Sigma = \frac{1}{n} (1 + \frac{ \sqrt{2} - 1}{n})^2 F^\top F$, which leads to 
\begin{align*}
W_2^2\left( \cL( x_t)   , \cL(y_t) \right)  & =\tfrac{n-1}{n} ( \sqrt{2} - 1)^2  \alpha^2_t  \quad \text{for all $t \in[T]$.} 
\end{align*}
For the matched setting $\Gamma = \Lambda$ and $\Sigma = \frac{1}{n} F^\top F$, the same result holds with the factor $(n-1)/n$  replaced by one.  

Combining Theorem~\ref{thm:couplingLB} with \eqref{eq:W2_sandwich} gives a general lower bound on the second moments of the error for any coupling of $X$ and $Y$. 

\begin{cor} 
Under Condition~\ref{cond:LB}, any coupling $(X, Y)$ satisfies the lower bound
\begin{align*}
\ex{ \norm{X - Y}^2 } & \ge \frac{n-1}{n} ( \sqrt{2} - 1)^2  \gtr\big(  ( \Id - \Lambda)^{-\top} \tfrac{1}{n} F^\top F  ( \Id - \Lambda)^{-1} \big).
\end{align*}
\end{cor}

These results show that the  quadratic Wasserstein distance is bounded away from zero for all $n$, and thus it is not possible for the coupling error to have vanishing second moments. The following example demonstrates that the optimal coupling error can scale exponentially with $T$  even for well-behaved Lipschitz continuous functions.

\begin{example}[Autoregressive Process]  Suppose that  $f_t$ are constant functions and  $g_t(x_{<t}) = \lambda \, x_{t-1}$ for  $\lambda \in \bbR$. Then,  $\Lambda $ and $(\Id - \Lambda)^{-1}$ are upper triangular matrices given by 
\begin{align*}
\Lambda = \begin{bmatrix} 
0 &  \lambda & 0 & \cdots & 0  \\
  & 0 & \lambda & \ddots & \vdots  \\
  &   & \ddots &\ddots  &0  \\
&&&0  & \lambda \\
   &  &&  &  0 
\end{bmatrix}, \qquad 
 ( \Id - \Lambda)^{-1}  = 
  \begin{bmatrix} 
1 & \lambda & \lambda^2   & \cdots & \lambda^{T-1}  \\
 & 1 & \lambda  & \cdots & \lambda^{T-2}  \\
& & 1 & \cdots & \lambda^{T-3}  \\
&&&\ddots \\
  &   &    &  &  1
\end{bmatrix} 
\end{align*}
Assuming $\Sigma = \Id$, it follows that 
\begin{align*}
\alpha_t & = \sum_{s=0}^{t-1} \lambda^{2 s} = \begin{dcases}  \frac{ \lambda^{2t} - 1}{ \lambda^{2} - 1}, &  \lambda \ne 1 \\
t & | \lambda| = 1
\end{dcases}
\end{align*}
This sequence grows exponentially with  $t$ for all $|\lambda| > 1$. 
\end{example}

\section{Proofs of Main Results}

\subsection{The Gaussian Conditioning Approach}\label{sec:gaussian_cond}
The Gaussian conditioning approach was introduced by  Bayati and Montanari~\cite[Lemma~10]{bayati:2011}  to show that the AMP recursion in \eqref{eq:amp} can be expressed as a conditionally Gaussian sequence --- that is the increments are Gaussian with a mean and variance that may depend on the past states. The key idea is to analyze the conditional distribution of the random matrix $A$ given the history $x_{<t}$, and then use this to determine the conditional distribution of the next iterate $x_t$.  A detailed derivation of this argument is presented in  Feng et al.~\cite[Appendix~A.2]{feng:2022}.  Related decomposition techniques have also appeared previously in the context of discrete-time filtering \cite{liptser:2001_applications}

In this section, we present a self-contained derivation of the Gaussian conditioning approach in a general setting, which includes the generalized-first order method in \eqref{eq:xt} as a special case.  We begin by recalling a basic but important fact: the conditional distributions of jointly Gaussian random variables are Gaussian, even in the presence of a degenerate (i.e., singular) joint covariance matrix.

\begin{lemma}[{$\!\!$\cite[Theorem~13.2]{liptser:2001_applications}}] \label{lem:z_given_x}
Let $(\theta, \xi)$ be jointly Gaussian random finite-dimensional vectors with means $m_x = \ex{ \theta}$ and  $m_\xi = \ex{ \xi}$ and covariance matrices $C_{\theta \theta} = \cov(\theta,\theta)$, $C_{\theta \xi} = \cov(\theta, \xi)$, and $C_{\xi \xi} = \cov(\xi,\xi)$. Then the conditional distribution of $\theta$ given $\xi$ is Gaussian with mean and covariance:
\begin{align*}
\ex{ \theta \mid \xi} = m_\theta  +C_{\theta \xi}  C_{\xi \xi }^+ ( \xi -  m_\xi ) , \qquad \cov(\theta \mid \xi) = C_{\theta \theta} - C_{\theta \xi} C_{\xi \xi }^+ C_{\xi \theta} 
\end{align*}
where $(\cdot)^+$ denotes the Moore-Penrose pseudoinverse. 
\end{lemma}

To proceed, we consider a generalization of \eqref{eq:xt} that allows for general linear operators applied to the matrix. Specifically, we consider the system 
\begin{align}
\xi_t = F_t \, \theta  +  g_t    \label{eq:xi}
\end{align}
where $\xi \in \bbR^n$ is the state at time $t \in \bbN$,  $\theta \sim \normal(0, \Id_N)$ is a standard Gaussian random vector, and $F_t \in \bbR^{n \times N}$ and $g_t \in \bbR^n$ are measurable functions of the states $(\xi_1, \dots, \xi_{t-1})$.  We define $\xi_{\le t} $, $F_{\le t}$ and $g_{\le t}$ to be the vertical stackings of the first $t$ terms, i.e.,  
\begin{align*}
\xi_{\le t} \coloneqq  \begin{bmatrix} \xi_1 \\ \vdots \\ \xi_t \end{bmatrix}  , \qquad 
F_{ \le t} \coloneqq  \begin{bmatrix} F_1 \\ \vdots \\ F_t \end{bmatrix}  , \quad \qquad g_{ \le t} \coloneqq  \begin{bmatrix} g_1 \\ \vdots \\ g_t \end{bmatrix}, 
\end{align*}

\begin{lemma}[Gaussian Conditioning]\label{lem:gaussian_cond} 
For each $ t \in \bbN$, the conditional distribution of $(\theta, \xi_t)$ given $\xi_{< t}$ is Gaussian with condition mean
\begin{align*}
\ex{ \theta  \mid \xi_{< t} } & = F_{< t}^+ ( \xi_{< t} - g_{< t}) \\
  \ex{ \xi_t \mid \xi_{ < t}}  & = F_t  F_{<t}^+ ( \xi_{< t} - g_{< t})  +g_t
\end{align*}
and conditional covariance 
\begin{align*}
 \cov( \theta, \theta  \mid \xi_{<  t} )  & = \Id  - F_{<  t}^+ F_{<  t} \\
  \cov( \theta , \xi_t \mid \xi_{< t} ) & = (  \Id  - F_{< t}^+ F_{ <  t})F_t^\top \\
 \cov( \xi_t , \xi_t \mid \xi_{< t} ) & = F_t (  \Id  - F_{< t}^+ F_{ <  t})F_t^\top.
\end{align*}
\end{lemma}

\begin{proof}
 For this proof we will verify that the  conditional distribution of $\theta$ given $\xi_{<t}$  is Gaussian with the stated mean and covariance. The result for the joint conditional distribution of $(\theta, \xi_t)$ given $\xi_{<t}$ then follows directly from \eqref{eq:xi} and the fact that $(F_t, g_t)$ are measurable functions of $\xi_{<t}$. 

We proceed by induction on $t$. The  base case ($t= 1$) follows immediately from noting that   $\xi_1 = F_1 \theta + g_1$ where  $F_1$ and $g_1$ are deterministic.

For the inductive step, assume that for time $t \in \bbN$, the conditional distribution of $\theta$ given $\xi_{<t}$ is  Gaussian with mean $F_{< t}^+ ( \xi_{< t} - g_{< t})$ and covariance $ (\Id-  F_{< t}^+  F_{< t}  )$.  This implies that  $(\theta, \xi_{< t})$ are jointly Gaussian, and thus there exists a decomposition of the form
\begin{align}
\theta  =F_{< t}^+  ( \xi_{< t} - g_{< t}) +   (\Id-  F_{< t}^+  F_{< t}  ) \theta'   
\label{eq:z_to_zp}
\end{align}
where $\theta' \sim \normal(0, \Id_N)$ is independent of $\xi_{<t}$. Introducing the notation
\begin{align*}
\bar{F}_t \coloneqq F_t  (\Id - F_{< t}^+  F_{< t}  ), \qquad  \bar{g}_t \coloneqq F_t F_{< t}^+   ( \xi_{< t} - g_{<t})  + g_t
\end{align*}
the next term in the sequence can be expressed as
\begin{align*}
\xi_{ t}  &=   F_t \,  \theta + g_t    
= \bar{F}_t \theta' + \bar{g}_t. 
\end{align*}
At this point, the crucial observation is that both  $\bar{F}_t$ and $\bar{g}_t$ are  measurable with respect to $\xi_{< t}$. Thus, we may condition on $\xi_{<t}$ and apply Lemma~\ref{lem:z_given_x} to the pair $(\theta', \xi_t)$ to see that 
\begin{align*}
\theta' \mid \xi_{\le t}  \sim \normal\left( \bar{F}_{ t}^+  (\xi_{t} - \bar{g}_{t}) , \Id -  \bar{F}_{ t}^+ \bar{F}_{ t} \right), 
\end{align*}
where we have used the property $M (M^\top M)^+ = M^+$ for any matrix $M$. 
Substituting back into  \eqref{eq:z_to_zp}, it follows that $\theta$ given $\xi_{\le t}$ is conditionally Gaussian with
\begin{subequations}
\begin{align}
\ex{ \theta  \mid x_{\le t} }& = F_{< t}^+  ( \xi_{< t} - g_{< t}) +   (\Id-  F_{< t}^+  F_{< t}  )  \bar{F}_{ t}^+  (\xi_{t} - \bar{g}_{t}) \label{eq:Ezgx_alt} \\
 \cov( \theta \mid \xi_{\le t} ) & =  (\Id-  F_{< t}^+  F_{< t}  )  ( \Id -  \bar{F}_{ t}^+ \bar{F}_{ t} )  (\Id-  F_{< t}^+  F_{<t}).  \label{eq:Czgx_alt} 
\end{align}
\end{subequations}

To complete the proof of the inductive step, it remains to verify that these expressions for the mean and covariance match the ones in the statement of the result. By construction, the row spaces of $F_{<t}$ and $\bar{F}_t$ are orthogonal and together span the row space of $F_{\le t}$. Hence, the orthogonal projection matrices onto these subspaces satisfy
\begin{align}
F_{< t}^+  F_{ < t}  +  \bar{F}_{ t}^+ \bar{F}_{ t}   = \bar{F}_{\le  t}^+ \bar{F}_{\le t}.  \label{eq:proj_sum} 
\end{align}
Using this relation,  the conditional variance in  \eqref{eq:Czgx_alt}  simplifies to  $ \cov( \theta \mid \xi_{\le t} )   = \Id -  F_{\le t}^+ F_{\le  t}$, as desired. 

The orthogonality of the row spaces of $F_{<t}$ and $\bar{F}_t$ also implies $ ( \Id - F^+_{<t}  F_{< t} )  \bar{F}_{ t}^+  = 0$ and  $\bar{F}_t^+ F_t  F_{< t}^+  =  0  $, which can be used to simplify the expression for the mean in \eqref{eq:Ezgx_alt}. In particular, we obtain
 \begin{align*}
\ex{ \theta \mid \xi_{\le t} }  
 = \begin{bmatrix}   F_{< t}^+ & \bar{F}_{ t}^+ \end{bmatrix}  (\xi_{\le t} - g_{\le t}). 
\end{align*}
Since  $\xi_{\le t} - g_{\le t}$ lies the column space of $F_{\le t}$, we can write
\begin{align*}
 \begin{bmatrix}   F_{< t}^+ & \bar{F}_{ t}^+ \end{bmatrix}  (\xi_{\le t} - g_{\le t})  
& = \begin{bmatrix}   F_{< t}^+ & \bar{F}_{ t}^+ \end{bmatrix}  F_{\le t} F_{\le t}^+  (\xi_{\le t} - g_{\le t})   \\
& =  ( F_{<t}^+ F_{<t}  + \bar{F}^+_t F_t )   F_{\le t}^+  (\xi_{\le t} - g_{\le t})   \\
& =  ( F_{<t}^+ F_{<t}  + \bar{F}_t^+ \bar{F}_t )   F_{\le t}^+  (\xi_{\le t} - g_{\le t})   \\
& = F_{\le t}^+   (\xi_{\le t} - g_{\le t}),   
\end{align*}
where the third step follows from the decomposition $F_t = \bar{F}_t + F_{t} F_{<t}^+ F_{< t}$, and the final step follows from \eqref{eq:proj_sum} and the identity $M^+ MM^+ = M^+$ for any matrix $M$.  Combining these displays yields  $\ex{ \theta \mid \xi_{\le t} } = F_{\le t}^+   (\xi_{\le t} - g_{\le t})$,  completing the proof of the inductive step. 
\end{proof}

\subsection{Proof of Coupling Construction (Theorem~\ref{thm:coupling})}
In this proof, we first show that construction in Definition~\ref{def:coupling} can be expressed as a special case of the general system in  \eqref{eq:xi}. We then use orthogonality properties of the matrix-valued functions to verify that $(z_t)$ has independent standard Gaussian increments. 

Define the normalized variables $u_t = n^{-1/2} q_t$ and the matrices  $D_t, E_t \in \bbR^{n \times n} $ according to 
\begin{align*}
 D_t \coloneqq \Id_n -  \frac{1}{2 } u_t u_t^\top   -    \sum_{s < t} u_s u_s^\top, \qquad E_t \coloneqq  \frac{1}{2 } u_t e_t^\top  +  \sum_{s < t} u_s e_s^\top,
\end{align*}
where $e_t$ denotes the $t$-th standard basis vector in $\bbR^n$. Then, the update in \eqref{eq:coupling_zt} can be rewritten as
\begin{align}
z_t & = \sqrt{n} \Big(   A u_t  + \frac{1}{2} ( A'_{tt} -   \langle u_t, A u_t \rangle  ) \,  u_t +   \sum_{s  < t} ( A'_{st} - \langle u_s, A u_t \rangle  )  \, u_s \Big)  \notag \\
& = \sqrt{n} \left(    D_t  A u_t  + E_t A' u_t   \right)  \notag \\
& =  \sqrt{n} \begin{bmatrix}  (u_t^\top \otimes  D_t )  &   ( e_t^\top  \otimes   E_t )   \end{bmatrix}  
\begin{bmatrix} \gvec(A) \\ \gvec(A') \end{bmatrix}  \label{eq:ztxt}
\end{align}
The vectorization $[A,A']$ can be constructed from a standard Gaussian vector  $\theta \sim \normal(0, \Id_N)$ of length $N  = 2n$ according to 
\begin{align*}
\begin{bmatrix} \gvec(A) \\ \gvec(A') \end{bmatrix} = \sqrt{\frac{2}{n} }  \begin{pmatrix} \Pi_n & 0 \\ 0 & \Pi_n \end{pmatrix}  \theta
\end{align*}
where $\Pi_n$ denotes the orthogonal projection  matrix on $\bbR^{n^2}$  satisfying $\Pi_n \gvec(M)  = \frac{1}{2} \gvec(M + M^\top)$ for any $M \in \bbR^{n \times n}$  \cite[pg.~56]{magnus:2007}. Plugging this expression back into \eqref{eq:ztxt}, we see that  \eqref{eq:coupling_zt} can be expressed in the form $z_t = F_t\,  \theta$ where 
\begin{align*}
F_t \coloneqq   \sqrt{2 } \begin{bmatrix}  (u_t^\top \otimes  D_t )  \Pi_n &   ( e_t^\top  \otimes   E_t )  \Pi_n  \end{bmatrix}  
\end{align*}
is a  measurable function of $(z_1, \dots, z_{t-1})$. 

By Lemma~\ref{lem:gaussian_cond}, the conditional distribution of $z_t$ given $z_{<t}$ is Gaussian with mean and variance specified by the matrix  sequence $(F_s)_{s \le t}$. To show that $(z_t)$ has independent standard Gaussian increments, we need to verify that the conditional mean is zero and the conditional variance is the identity matrix.  For $1 \le s \le  t$ we  use the properties  $\Pi^2_n = \Pi_n$ and 
\begin{align*}
(v_1^\top \otimes M_1) \Pi_n (v_2  \otimes M_2^\top)  = \frac{1}{2}  v_1^\top v_2 M_1 M_2^\top +  \frac{1}{2} M_1 v_2  v_1^\top M_2^\top 
\end{align*}
for all $v_1, v_2 \in \bbR^k,  M_1, M_2 \in \bbR^{m \times k} $ and the fact that $\sum_{t = 1}^{T} u_t u_t^\top  = \Id$ to see that, with probability one, 
\begin{align*}
F_s F_t^\top & =  2   (u_s^\top  \otimes  D_s )  \Pi_n  (u_t \otimes  D_t )  +   2 ( e_s^\top  \otimes   E_s )  C_T     ( e_t  \otimes   E_t^\top )  \\
 & =   u^\top_s  u_t   D_s D_t  + D_s u_t   u_s^\top D_t    +  e^\top_s  e_t   E_s^\top  E_t  +   E_s e_t  e_s^\top  E_t^\top \\
 & =  \one_{s = t}    \left[    D^2_t  + D_t u_t   u_t^\top D_t    +    E_t   E_t^\top  +   E_t e_t  e_t^\top  E_t^\top  \right] \\
  & =  \one_{s = t}  \Big[   \tfrac{1}{4 } u_t u_t^\top + \sum_{s > t} u_s u_s^\top     +   \tfrac{1}{4 } u_t u_t^\top  +  \sum_{s < t} u_s u_s^\top +  \tfrac{1}{4 } u_t u_t^\top   \Big]   \\
&   = \one_{ s = t}  \Id_n. 
\end{align*}
Hence, the row spaces of $F_1, \dots, F_t$ are orthogonal, and it follows that $z_t \sim \normal(0, \Id_n)$ is independent of $z_{<t} $. The condition $\Sigma = \Omega^\top \Omega$ ensures that   $(w_t)_{t \in [T]}$ is a Gaussian process with mean zero and covariance defined by $\Sigma = \Omega \Omega^\top$, and this verifies that $(y_t)_{t \in [T]}$ has the distribution specified by \eqref{eq:yt}.

The remaining step to verify the identity in \eqref{eq:ztqt_coupling}. We use \eqref{eq:coupling_zt} to see that
\begin{align*}
\langle z_s  , q_t \rangle 
= \begin{dcases}  \langle q_s, A q_t \rangle    & s  < t\\
\frac{1}{2} ( \langle q_t , A q_t \rangle + \sqrt{n} A'_{tt})   & s = t\\
n A'_{st}  &   s > t
\end{dcases}
\end{align*}
Making these substitutions, leads to 
\begin{align*}
z_t & = A q_t  + \frac{1}{2} ( A'_{tt} -  \tfrac{1}{n} \langle q_t, A q_t \rangle  ) \,  q_t +   \sum_{s  < t} ( A'_{st} -  \tfrac{1}{n} \langle q_s, A q_t \rangle  )  \, q_s \\
& = A q_t  +  ( A'_{tt} -  \tfrac{1}{n} \langle q_t, A q_t \rangle  )\,  q_t +   \sum_{s  < t} (A'_{st}  - \tfrac{1}{n} \langle z_s  , q_t \rangle   )  \, q_s \\
& = A q_t  -  \frac{1}{n} \sum_{s \le t} \langle  z_s  , q_t \rangle \,  q_s   +  \sum_{s  \le t} A'_{st}  q_s
\end{align*}
Since $\frac{1}{n} \langle  z_r  , q_t \rangle = A'_{rt}$ for all $r > t$,  the equality remains valid if these terms are included in the summations. \qed

\subsection{Proof of General Upper Bound (Theorem~\ref{thm:Delta12})}

In this proof, we bound the coupling error $\|X - Y\|$ under the construction given in Definition~\ref{def:coupling}. The proof proceeds in four main steps.

\paragraph{Stability Bound.}

Let $h_t = A f_t + g_t$ denote the full update function for the generalized first-order method. By the assumption that  $f_t$ and $g_t$ are Lipschitz with constants $L_f$ and $L_g$, we have
\begin{align*}
\Lip(h_t) \le  L_h  \coloneqq  \normop{A} L_f + L_g. 
\end{align*}
Applying the discrete-time stability bound in Lemma~\ref{lem:stability2} yields
\begin{align}
\norm{X- Y} \le   (1+ L_h)^{T-1}  \norm{Y -  A F(Y) - G(Y)}_{2,1} .
 \label{eq:XYtoDelta}
\end{align}
 Thus, the problem has been reduced to bounding the random variables $\normop{A}$ and  $\norm{Y -  A F(Y) - G(Y)}_{2,1} $.

\paragraph{Decomposition of Coupling Error.}

Let $Q = [q_1, \dots, q_T]$ and $Z = [z_1, \dots, z_T]$ be the $n \times T$ matrices generated by
\eqref{eq:coupling_qt} and \eqref{eq:coupling_zt}, respectively. Recall that the columns of $Q$ are obtained by Gram-Schmidt orthogonalization to the columns  $F$ with normalization  $\norm{q_t} = \sqrt{n}$. Thus, we have the QR decomposition $F(Y)= QR $  where $R$ is a $T \times T$  upper triangular matrix with non-negative diagonal entries defined by  $ R \coloneqq \tfrac{1}{n} Q^\top F$ 

Using the identity derived from \eqref{eq:ztqt_coupling} with $r =T$, we have 
\begin{align*}
Z = A Q -  \tfrac{1}{n} Q  Z^\top Q   + Q  A'',
\end{align*}
where  $A'' \in \bbR^{T \times T} $ is leading principal submatrix of $A'$. Multiplying both sides on the right by $R$ gives
\begin{align}
Z R = A F(Y) -  \tfrac{1}{n} Q  Z^\top Q R   + Q  A''  R  \label{eq:ZQ_coupling2} 
\end{align}
Comparing with $Y = M(Y) + Z \Omega$, we obtain 
\begin{align*}
Y - A F(Y)  - G(Y) &  =M(Y) - G(Y)     - \tfrac{1}{n} Q  Z^\top Q R  +  Z(\Omega - R)  + Q  A''  R.
\end{align*}
We now define the error terms: 
\begin{align*}
\Delta_1 & = \norm{M(Y) - G(Y) - F(Y) B_n}_{2,1} \cdot \norm{\Omega}_{2,1}^{-1}, \qquad \widetilde{\Delta}_2  = \normop{ R \Omega^{-1} - \Id}
\end{align*}
where we recall that $B_n = \frac{1}{n} \Omega^{-1} Z^\top QR$.

Using the triangle inequality, the bound $\norm{AB}_{2,1} \le \normop{A}\norm{B}_{2,1}$ for conformable matrices $A$ and $B$, and the fact that $\normop{Q} = \sqrt{n}$ , we can write
\begin{align}
\norm{Y - A F(Y)   - G(Y) }_{2,1}  &  \le \norm{ M - G     - F B_n}_{2, 1}  + \tfrac{1}{n}  \norm{ Q  ( R  \Omega^{-1}  - \Id) Z^\top  Q R }_{2, 1}  \notag  \\
& \quad   +  \norm{ Z(\Omega - R)}_{2,1}   + \norm{Q  A''  R}_{2,1} \notag \\
&  \le \Delta_1 \,  \norm{\Omega}_{2,1}  + \widetilde{\Delta}_2 \normop{Z}  \,   \norm{R}_{2,1}    +   \widetilde{\Delta}_2   \normop{Z}\,  \norm{ \Omega}_{2,1}   +\sqrt{n}  \normop{A''}   \norm{R}_{2,1}  \notag  \\
& \le \big( \Delta_1 +  \widetilde{\Delta}_2(2+ \widetilde{\Delta}_2) \normop{Z}  + \sqrt{n} (1 + \widetilde{\Delta}_2) \normop{A''}  \Big)  \cdot \norm{\Omega}_{2,1}.   \label{eq:DeltatoDelta12}
\end{align}

\paragraph{High Probability Bounds.} 
We employ standard bounds on the operator norms of the Gaussian random matrices.  The expectations satisfy $\ex{ \normop{Z}} \le \sqrt{n} + \sqrt{T}$ by  \cite[Theorem~II.13]{davidson:2001} and $ \ex{ \normop{A}} \le 2$ and $ \ex{   \normop{A''}} \le 2 \sqrt{T/n}$ by  \cite[Theorem~II.11]{davidson:2001}. 
By Gaussian concentration for Lipschitz functions (Lemma~\ref{lem:TIS}), it follows that, for all $r \ge 0$, 
\begin{align*}
 \pr*{ \sqrt{n} \, \normop{A} \ge 2  \sqrt{n} + \sqrt{2 r } }  & \le e^{ -r}\\
 \pr*{ \sqrt{n} \, \normop{A''} \ge 2 \sqrt{T}  + \sqrt{2 r} }  & \le e^{ -r}\\
\pr*{ \normop{Z} \ge \sqrt{n} + \sqrt{T}  + \sqrt{ 2r} }&  \le e^{-r}. 
\end{align*}
Combining these bounds with  \eqref{eq:XYtoDelta} and  \eqref{eq:DeltatoDelta12}, we find that the event
 \begin{align*}
\norm{X- Y} & \le \left(1+   (2 + \sqrt{2r/n}) L_f + L_g  \right )^{T-1} \\
& \times \Big( \Delta_1   +\widetilde{\Delta}_2(2+ \widetilde{\Delta}_2) ( \sqrt{n} + \sqrt{T} + \sqrt{2r} )   + (1 + \widetilde{\Delta}_2)  (  2 \sqrt{T}  + \sqrt{2 r} )   \Big)  \, \norm{\Omega}_{2,1} 
\end{align*}
holds with probability at least $1- 3e^{-r}$. Restricting this inequality to the setting $r \le n$ and using that $T \le n$ gives the simplified bound
 \begin{align*}
\norm{X- Y} & \le \left(1+   (2 + \sqrt{2})  L_f + L_g  \right )^{T-1} \Big(  \Delta_1   + (2 + \sqrt{2}) (3\widetilde{\Delta}_2+ \widetilde{\Delta}^2_2)   \sqrt{n}     +   2 \sqrt{T}  + \sqrt{2 r}    \Big)  \, \norm{\Omega}_{2,1}.
\end{align*}

 \paragraph{Perturbation Bounds for Cholesky Decomposition.}  The term $\widetilde{\Delta}_2$ measures the difference between the Cholesky factor $R$ of the empirical covariance matrix $\frac{1}{n} F(Y)^\top F(Y)$ and the target matrix $\Omega$. By construction, both of these matrices are upper triangular and have non-negative diagonal entries. Since the set of upper triangular matrices with non-negative diagonal entries is closed under inversion and multiplication the product $R \Omega^{-1} $ also has these properties. Thus, we can apply Lemma~\ref{lem:chol_pert} along with $S_n =  \Omega^{-\top} R^\top R \Omega^{-1}$ to see that 
 \begin{align*}
3\widetilde{\Delta}_2+ \widetilde{\Delta}^2_2  \le  15 \log_2(4T) \normop{ S_n - \Id} . 
 \end{align*}
Using that $ (2 + \sqrt{2}) 15  < 52$, the stated bound holds with $\Delta_2 = 52 \log_2(4T) \sqrt{n} \normop{ S_n - \Id}$.

\subsection{Proof of Refined Upper Bound (Theorem~\ref{thm:XY})}

In this proof, we use the Lipschitz continuity assumptions to simplify the general bound on the coupling error in Theorem~\ref{thm:Delta12}. In particular, we establish tail bounds on the random variables 
\begin{align*}
\Delta_1 & \coloneqq     \norm{M(Y) - G(Y)-   F(Y)   B_n }_{2,1}   \cdot \norm{ \Omega}_{2,1}^{-1}  \\
\Delta_2 & \coloneqq  52 \,  \log_2(4T) \,   \sqrt{n}  \cdot     \normop{S_n  - \Id } 
\end{align*}
where $B_n \coloneqq \frac{1}{n}  \Sigma^{-1}  ( Y - M(Y))^\top F(Y)$ and $S_n \coloneqq  \frac{1}{n}  \Omega^{-\top} F(Y)^\top F(Y) \Omega^{-1}$.

\paragraph{Change of Variables.} 
The crucial step in our analysis is to show the process $(y_t)$ generated by system \eqref{eq:yt} can be expressed as Lipschitz continuous function of i.i.d.\ standard Gaussian variables. Let $\phi_t \colon (\bbR^{n} )^{t} \to \bbR^n$  be the mapping from $w_{\le t}$  to the state variable $y_t$ defined by \eqref{eq:yt}  such that  $y_t = \phi_t( w_{<t})$, and let  $\phi_{\le t}  \colon  (\bbR^{n} )^{t} \to (\bbR^n)^t$ denote the mapping from  $w_{\le t}$ to $y_{\le t}$. It is straightforward to see that this mapping defines a bijection: 
 \begin{align*}
 y_{\le t} = \phi_{\le t}( w_{\le t}) \qquad \iff \qquad y_{\le t} - m_{\le t}(y_{<t}) = w_{\le t} 
 \end{align*}
Under the assumption $\Lip(m_t)  \le L_m$  for all $t \in[T]$, Lemma~\ref{lem:stability1}  implies that $\phi_{\le t} \colon (\bbR^n)^{t} \to (\bbR^{n})^t$ is Lipschitz continuous with constant
\begin{align}
\Lip(\phi_{\le t})  \le  L_\phi \coloneqq  \sqrt{T} (1 + L_m)^{T-1}.   \label{eq:Lphi}
\end{align}

We also recall that the $n \times T$ matrix $Y = [ y_1, \dots, y_T]$ satisfies 
\begin{align*}
Y = M(Y) + W , \qquad W = Z \Omega
\end{align*}
where $Z  = [z_1, \dots z_T]\in \bbR^{n \times T}$ has independent standard Gaussian entries. From \eqref{eq:Lphi} it follows that $Y$ can be expressed as a Lipschitz continuous function of $Z$ with Lipschitz constant $L_\phi \normop{\Omega}$.

\paragraph{Bounds via Gaussian Concentration.} 
Armed with \eqref{eq:Lphi}, we can now apply Gaussian concentration inequalities to bound on the coupling error. We further decompose $\Delta_1$  using 
\begin{align*}
\Delta_1 \le  \underbrace{ \norm{ M(Y) - G(Y) - F(Y) B }_{2,1}  \cdot \norm{\Omega}_{2,1}^{-1}}_{= \Delta_{1a} } + \underbrace{ \norm{ F(Y) ( B_n -B) }_{2,1}   \cdot \norm{\Omega}_{2,1}^{-1}}_{ = \Delta_{1b} }.
\end{align*}
where we recall that $B = \ex{B_n}$ and $S = \ex{ S_n}$. We will establish that, for all $0 \le r \le n$, the following events hold with probability at least $1 - 2 e^{-r}$: 
\begin{align}
 \Delta_{1a}  &\le \psi_1    +(L_m + L_g + L_f) L_{\phi}  \kappa( \Omega)  \sqrt{r}   +  L_f L_{\phi}  \kappa( \Omega)  \,\psi_2  \label{eq:Delta1a_bound}\\
 \Delta_{1b} & \le c_1  \left( L_H \,  \psi_2+   L^3_H  \sqrt{r +T} \right )   \label{eq:Delta1b_bound}\\
 \Delta_{2} &  \le c_2 \log_2(2T)    \left(   \psi_2 +   L^2_H  \sqrt{r +T} \right )   \label{eq:Delta2_bound}
\end{align}
Here,  $c_1, c_2$ are universal positive constants and  $L_H = 1 +  L_f L_\phi \kappa(\Omega)$ where $L_\phi$ is given by \eqref{eq:Lphi} and $\kappa(\Omega)$ is the condition number of $\Omega$.

To simplify these bounds, we set $L= \log_2(T)  + \sqrt{T} (L_f + L_g + L_m)(1 + L_m)^{T-1} \kappa(\Omega)$. Combining the terms, we find that 
\begin{align*}
\Delta_1   + \Delta_2 &  \le \psi_1   + c_3 \, L   \psi_2  +c_4 \,  L^2 (\sqrt{T} + \sqrt{r})
\end{align*}
where $c_3, c_4$ are universal positive constants. Combining these bounds with Theorem~\ref{thm:Delta12} gives the desired bound on the coupling error $\norm{X-Y}$.

\paragraph{Proof of \eqref{eq:Delta1a_bound}.} 
We can write
\begin{align*}
\Delta_{1a} & = \norm{ M(Y ) - G(Y ) - F( Y ) B }_{2,1}  \cdot \norm{\Omega}_{2,1}^{-1}\\
& \le \normop{\big( M(Y ) - G( Y ) - F(Y ) B \big )  \Omega^{-1}  }\\
& \le \normop{ M(Y)  \Omega^{-1}  } + \normop{ G(Y ) \Omega^{-1}  } + \normop{ F(Y)  \Omega^{-1}  } \normop{\Omega B \Omega^{-1}}
\end{align*}
From this upper bound and \eqref{eq:Lphi}, it follows that  $\Delta_{1a}$ is a Lipschitz function of the i.i.d.\ Gaussian matrix $Z$ with Lipschitz constant 
\begin{align*}
L_{1a} \coloneqq (L_m + L_g + \normop{\Omega B \Omega^{-1}}  L_f) L_{\phi}  \kappa( \Omega) 
\end{align*}
The Gaussian concentration inequality Lemma~\eqref{lem:TIS}  leads to
\begin{align*}
\pr[\Big]{  \Delta_{1a}  \ge \ex{ \Delta_{1a}}    +L_{1a} \sqrt{r}  }  \le e^{-r} , \qquad \forall r \ge 0 
\end{align*}
By the Cauchy–Schwarz inequality
\begin{align*}
\normop{\Omega B \Omega^{-1}} & =\tfrac{1}{n}   \normop{ \ex{ Z^\top F(Y)^{-1} \Omega^{-1}}} \\
& \le  \tfrac{1}{n}  \sqrt{\normop{ \ex{ Z^\top Z} } \normop{ \ex{ \Omega^{-\top} F(Y)^\top F(Y) \Omega^{-1} }}}\\
& = \sqrt{ \normop{S}} \le  \sqrt{ 1 + \normop{S -\Id}} \le 1 +  \normop{S- \Id}.
\end{align*} 
Combining the above displays and restricting to $0 \le r \le n$ leads to \eqref{eq:Delta1a_bound}.

\paragraph{Proofs of \eqref{eq:Delta1b_bound} and \eqref{eq:Delta2_bound}.} 
We introduce the $n \times 2T$ random matrix $
H \coloneqq  [ F(Y) \Omega^{-1}  , Z]$ and observe that
\begin{align*}
\frac{1}{n}   H^\top H = \begin{bmatrix} S_n &[ \Omega B_n \Omega^{-1} ]^\top  \\    \Omega B_n \Omega^{-1}   &  \frac{1}{n} Z^\top Z \end{bmatrix}. 
\end{align*}
 By \eqref{eq:Lphi},   $H$ is a Lipschitz function of $Z$ with Lipschitz constant $L_H = 1 +  L_f L_\phi \kappa(\Omega)$.  
Thus, we can apply Lemma~\ref{lem:HHLip} to see that the event
\begin{align}
\normop{H^\top H - \ex{ H^\top H}} 
& \le 4  \normop{ \ex {H^\top H}}^{1/2}  L_H  \sqrt{2r}  + 
 2  L_H^2 (2 r+1)  \label{eq:Hr}
\end{align}
holds with probability at least $1 - 2 \cdot 9^{2 T} e^{-r}$ for all $r \ge 0$.

To bound $\Delta_{1b}$, set $u = L_H\sqrt{ 2r + 1}$ and observe that \eqref{eq:Hr} implies
\begin{align*}
\normop{H } & = \sqrt{ \normop{H^\top H } } \\
& \le   \sqrt{ \normop{ \ex{ H^\top H}}  +  4 \normop{\ex{H^\top H} }^{1/2} u  + 2 u^2 }\\
& \le \normop{ \ex{ H^\top H}}^{1/2}  +2  u.
\end{align*}
Accordingly, we can write
\begin{align*}
\Delta_{1b}  & = \norm{F(Y) (B_n - B) }_{2,1} \cdot \norm{\Omega}_{2,1}^{-1} \\
& = \normop{F(Y) \Omega^{-1} } \normop{ \Omega  (B_n - B) \Omega^{-1}  } \\
& \le  \frac{1}{n}   \normop{H} \normop{H^\top H - \ex{ H^\top H}}\\
& \le  \frac{1}{n}  \left(  \normop{ \ex{ H^\top H}}^{1/2} +2  u  \right)  \left(  4  \normop{\ex{H^\top H} }^{1/2}   + 2 u\right) u  \\
& \le  \frac{8}{n}  \left(  \normop{ \ex{ H^\top H}}   +  u^2  \right)  u\\
& \le 8\left(  \normop{ \Id - S} + 2 +  \frac{L_H^2 (2 r + 1)}{n}  \right)   L_H\sqrt{2r+1} 
\end{align*}
Replacing $r$ with $ 1 + 2 T \ln(9) + \ln(2)  + r$ and then restricting to $r \le n$  and $T \le n$ gives to the bound in \eqref{eq:Delta1b_bound}.

To bound  $\Delta_2$, observe that
\begin{align*}
 \normop{\ex{H^\top H} }^{1/2}  & \le  \sqrt{ n  \normop{ S } +   \normop{ \ex{Z^\top Z}  } }   =  \sqrt{  n}   \normop{S- \Id}^{1/2} + \sqrt{ 2n}  
\end{align*}
Combining with  \eqref{eq:Hr} leads to 
\begin{align*}
\normop{S_n  - \Id}  
&\le   \normop{S  - \Id}  +   \normop{S_n - S}\\
 &\le   \normop{S  - \Id}  +  \frac{1}{n}  \normop{H^\top H - \ex{ H^\top H}}\\
 & \le  \normop{S  - \Id}  +  4   \normop{S - \Id}^{1/2} \,  L_H  \sqrt{\frac{2r}{n}}  + 8\,  L_H  \sqrt{\frac{r}{n}}  + 
 \frac{ 2  L_H^2 (2 r+1)}{ n}  \\
  & \le  2 \normop{S  - \Id}  + 8 L_H \sqrt{ \frac {r}{ n}}   + 
 \frac{ 2  L_H^2 (6 r+1)}{ n},  
\end{align*}
where the last step follows from the basic inequality $ab \le a^2 + b^2$. Replacing $r$ with $ 1 + 2 T \ln(9)  + \ln(2) + r$ and then restricting to $r \le n$  and $T \le n$ gives to the bound in \eqref{eq:Delta2_bound}.

\subsection{Proof of Lower Bound (Theorem~\ref{thm:couplingLB}) }

Under the theorem assumptions, the matrices $X$ and $Y$ satisfy the fixed point equations
\begin{align}
X = AF + X \Lambda, \qquad Y = Y \Gamma  + Z \Omega
\end{align}
where:
\begin{itemize}
\item $F\in \bbR^{n \times T}$ satisfies  $\frac{1}{n} F^\top F = \Sigma = \Omega^\top \Omega$,
\item $\Lambda \in \bbR^{T \times T}$ and $ \Gamma  \in \bbR^{T \times T}$ are strictly upper triangular,
\item $A \in \bbR^{n \times n}$ is drawn from $\mathsf{GOE}(n)$, and 
\item $Z \in \bbR^{n \times T}$ has independent standard Gaussian entries.
\end{itemize}
 The fact that $\Lambda$  and $\Gamma$ are strictly upper triangular ensures that $ \Id - \Lambda$ and $\Id  - \Gamma$ are invertible. Solving  for $X$ and $Y$ gives 
\begin{align*}
X =  AF( \Id - \Lambda)^{-1}  \qquad Y =   Z \Omega ( \Id - \Lambda)^{-1}, 
\end{align*}
From these solutions, we see that $X$ and $Y$ both have zero-mean matrix-variate Gaussian distributions. To express their covariance, we use vectorization to write
\begin{align*}
\gvec(X) &= ( (\Id -\Lambda)^{-\top} F^\top \otimes \Id_n)   \gvec(A)  \\
 \gvec(Y)  &=  (  ( \Id - \Lambda)^{-\top} \Omega^\top \otimes \Id_n) \gvec(Z)  
\end{align*}
Noting that  $\cov(\gvec(A)) = \frac{2}{n} \Pi_n$,  where   $\Pi_n =  \frac{1}{2} ( \Id_{n^2} + \Comm_{n,n})$ is the matrix representation of the orthogonal projection operator on the space of $n \times n$ symmetric matrices,  and $\cov(\gvec(Z)) = \Id_{n^2}$ leads to 
 \begin{align*}
\cov( \gvec(X) ) &=  \frac{2}{n} ( (\Id -\Lambda)^{-\top} F^\top \otimes \Id_n)   \Pi_n  ( F(\Id -\Lambda)^{-1}  \otimes \Id_n)\\
\cov(  \gvec(Y) )   &=   (  ( \Id - \Lambda)^{-\top} \Sigma ( \Id - \Lambda)^{-1}  \otimes \Id_n) 
\end{align*}
At this point, we appeal to the closed-form expression for the quadratic Wasserstein distance between multivariate Gaussians \cite{dowson:1982}, which is given by 
\begin{align}
W_2^2\left(\normal(\mu_X, C_X)  , \normal(\mu_Y, C_Y)\right)  = \norm{ \mu_X - \mu_Y}^2 +  \gtr\Big( C_X + C_Y - 2 ( C_Y^{1/2} C_X C_Y^{1/2} )^{1/2}   \Big) \label{eq:W2_normal}
\end{align}
where $(\cdot)^{1/2}$ denotes the symmetric positive semidefinite square root. 

To simplify the analysis, we focus on the individual columns of $X$ and $Y$. For each $t \in [T]$ we can write
 \begin{align*}
\cov( x_t) &=  \frac{2}{n} ( e^\top_t (\Id -\Lambda)^{-\top} F^\top \otimes \Id_n)   \Pi_n  ( F (\Id -\Lambda)^{-1} e_t   \otimes \Id_n)\\
&=    [ ( \Id - \Lambda)^{-\top} \tfrac{1}{n} F^\top F  ( \Id - \Lambda)^{-1} ]_{tt} \cdot \Id_n + 
\tfrac{1}{n}  F (\Id -\Lambda)^{-1} e_t  e_t^\top (\Id -\Lambda)^{-\top} F^\top \\
& = \alpha^2_t (\Id_n + u_tu_t^\top  ) \\
\cov( y_t )   &=   \beta^2_t \Id_n 
\end{align*}
where $u_t$ is a unit vector in $\bbR^n$, $\alpha^2_t =  [ ( \Id - \Lambda)^{-\top} \tfrac{1}{n} F^\top F  ( \Id - \Lambda)^{-1} ]_{tt}$, and $\beta^2_t =  [ ( \Id - \Gamma)^{-\top} \Sigma ( \Id - \Gamma)^{-1} ]_{tt}$.  Using \eqref{eq:W2_normal} and  $(\Id  + u_t u_t^\top)^{1/2} =   \Id   + (\sqrt{2} - 1)  u_tu_t^\top $,  it follows that 
\begin{align*}
W_2^2\left( \cL( x_t)   , \cL(y_t) \right)  & =  \gtr\Big( \alpha_t^2 ( \Id + u_t^\top u_t)  + \beta_t^2 \Id_n  - 2 \alpha_t \beta_t  (  \Id   + (\sqrt{2} - 1)  u_tu_t^\top )  \Big) \\
& = (n+1)   \alpha_t^2   + n \,  \beta_t^2 -  2 (n  + \sqrt{2} - 1) \, \alpha_t \beta_t \\
& = \tfrac{ n-1}{n} ( \sqrt{2} - 1)^2  \alpha^2_t  + n\,  \Big( ( 1 + \tfrac{ \sqrt{2} - 1}{n}) \alpha_t - \beta_t \Big)^2 .
\end{align*}
This concludes the proof of Theorem~\ref{thm:couplingLB}. \qed

\section{Further Technical Results}

\subsection{Gaussian Concentration}

We first recall some basic definitions and properties of sub-Gaussian distributions; see e.g.,  \cite[Section~2.2]{boucheron:2013} for more details. A real-valued random variable $V$ is sub-Gaussian with variance proxy $\sigma^2$ if its cumulant generating function satisfies
\begin{align*}
\log \ex{ e^{ \lambda (V- \ex{V}) }} \le  \frac{ \lambda^2 \sigma^2}{ 2} , \quad \text{for all $\lambda \in \bbR$.}
\end{align*}
Among other things, this condition implies that $\var(V) \le \sigma^2$ and $\pr{ V  \ge \ex{V}  +  \sigma \sqrt{2 r}} \le e^{-r}$ for all $r \ge 0$. 

The concentration of measure phenomenon for Gaussian measure provides a tight control on the behavior of functions of Gaussian functionals under  various smoothness conditions. In particular,  the  Tsirelson--Ibragimov--Sudakov Inequality \cite[Theorem~5.5]{boucheron:2013} states that an $L$-Lipschitz function of independent standard Gaussian variables is sub-Gaussian with variance proxy $L^2$. This leads to the following concentration inequality: 

\begin{lemma}[Gaussian Concentration Inequality {\cite[Theorem~5.6]{boucheron:2013}}]\label{lem:TIS} Let  $z \sim \normal(0 , \Id_N)$ be an $N$-dimensional vector with independent standard Gaussian entries and let $f\colon \bbR^N \to \bbR$  an $L$-Lipschitz continuous function. Then, for all $r \ge 0$, 
\begin{align*}
\pr*{ f(z) \ge \ex{ f(z)}  + L \sqrt{2r}} \le e^{-r}. 
\end{align*}
\end{lemma}

The next result provides a concentration bound for a random matrix $H$ with the property that any linear combination of the columns is sub-Gaussian. Note that this condition is satisfied automatically if there exits an $L$-Lipschitz function  $f\colon \bbR^N \to \bbR^{n \times d}$ such that $H = f(z)$ for $z \sim \normal(0, \Id_N)$. 

\begin{lemma}\label{lem:HHLip} 
Let  $H \in \bbR^{n \times d}$ be a random matrix with the property that $\norm{H u}$ is sub-Gaussian with variance proxy $L^2$ for every unit vector $u$ in $\bbR^d$. 
Then, for all $r \ge 0$, 
\begin{align*}
\pr[\Big]{  \normop*{H^\top H - \ex{ H^\top H}} \ge  4    \normop{ \ex {H^\top H}}^{1/2} \,  L   \sqrt{2r}  + 
 2  L^2 (2 r+1)  } \le   2\cdot  9^d e^{- r}.
\end{align*}
\end{lemma}

\begin{proof}
By a standard $1/4$-net covering argument (see Exercise 4.4.5 (b) and Corollary 4.2.13 in \cite{vershynin:2018}), the operator norm of a $d \times d$ symmetric matrix $A$ can be bounded from above according to 
\begin{align*}
\normop{A} \le 2  \sup_{u \in \cU} | \langle u , A u \rangle |,
\end{align*}
where $\cU$ is a subset of unit sphere in $\bbR^d$ with cardinality at most $9^d$. 

We will apply this inequality to the matrix  $A = H^\top H  - \ex{ H^\top H}$. For each $u \in \cU$, we can write
\begin{align*}
| \langle u , A u \rangle | & = \big |  \| H u \|^2 - \ex{ \|H u\|^2} \big |\\
& \le  \big |  \| H u \|^2 -  \ex{ \|H u\|}^2 \big |  + \var( \norm{Hu})\\
&\le  2  \ex{ \|H u\|}  \big |  \| H u \| -  \ex{ \|H u\|} \big |  +  \big |  \| H u \| -  \ex{ \|H u\|} \big |^2   + \var( \norm{Hu} )
\end{align*} 
By Jensen's inequality, $ \ex{\norm{H u}}^2  \le\ex{ \norm{H u}^2} = u^\top \ex{ H^\top H } u \le \normop{ \ex{ H^\top H}}$. Furthermore, by the assumptions that $\norm{Hu}$ is sub-Gaussian with variance proxy $L^2$, it follows that 
\begin{align*}
\var(\norm{Hu}) \le L^2, \qquad \pr{ \big| \norm{Hu}  - \ex{\norm{Hu}} \big| \ge  L \sqrt{ 2r} } \le 2 \,  e^{- r} 
\end{align*} 
for all $r \ge 0$. 
the random variable $\norm{Hu}$. Combining these displays with the union bound gives the stated result. 
\end{proof}

\subsection{Matrix Perturbation Bounds}

The Cholesky factorization of a symmetric positive semidefinite matrix $S$ is the matrix decomposition $S= U^\top U$ where $U$ and an upper triangular matrix with non-negative diagonal entries. If $S$ is positive definite then the decomposition is unique and the diagonal entries are strictly positive. The next result is a variation on the perturbation bound given by Edelman and Mascarenhas~\cite{edelman:1995}. 

\begin{lemma}[{$\!\!$\cite{edelman:1995}}]\label{lem:chol_pert} 
Let $U$  be a $T \times T$ upper triangular matrix ($T \ge 2$) with non-negative diagonal entries. 
 The following inequalities hold:
\begin{align*}
\normop{U-\Id} & \le 2 \,   \log_2(4T)  \,   \normop{U^\top  U - \Id }  \\
\normop{U-\Id}^2 & \le 9 \,   \log_2(4T)  \,   \normop{U^\top  U - \Id }  
\end{align*}
\end{lemma}
\begin{proof}
The first inequality follows directly from \cite[Equation (1.2)]{edelman:1995} so all that remains is to prove the second inequality.  If $T = 1$ and then the desired result follows from the basic inequality $|x-1|^2 \le |x^2 - 1|$ for all $x \in \bbR$. For $T \ge 2$, 
we consider two cases. First, if $2 \log_2(4T)   \normop{ U^\top U - \Id} \le 1$ then $\normop{ U - \Id } \le 1$ and so $\normop{U- \Id}^2 \le \normop{U-\Id} \le  9 \,   \log_2(4T)  \,   \normop{U^\top  U - \Id } $. Alternatively, if $2 \log_2(4T)   \normop{ U^\top U - \Id}  > 1$ then we can use the bound
 \begin{align*}
  \normop{U-\Id}& \le 1 + \normop{U}  = 1  + \sqrt{ \normop{U^\top U}  } \le   1+ \sqrt{ 1+ \normop{U^\top U - \Id}} \\
  & <  ( 1+ \sqrt{7/6} ) \sqrt{ 2\log_2(4T)  \normop{ U^\top U - \Id}}.
 \end{align*}
 Squaring both sides and noting that  $(1 + \sqrt{7/6})^2 2 < 9$ gives the desired result. 
\end{proof}

\subsection{Stability of Discrete-Time Systems}

Let $\cX$ be a Hilbert space and consider the system 
\begin{align}
v_t =h_t(v_{< t} )  + u_t , \label{eq:ht} 
\end{align}
where $v_t \in \cX$ is the state at time $t \in \bbN$, $h_t \colon \cX^{t-1} \to \cX$  is measurable function of previous states, and $u_t  \in \cX$ is a control term. Let $\phi_t \colon \cX^t \to \cX$ be mapping from from the control sequence $u_{\le t}$ to the state $v_t$ such that $v_t  =  \phi_t(u_{\le t})$. Let $\phi_{\le T} \colon \cX^T \to \cX^T$ be the mapping from $u_{\le T}$ to  $v_{\le T}$ such that $v_{\le T} = \phi_{\le T}(u_{\le T})$. 

The following results provide bounds on the stability of \eqref{eq:ht} 

\begin{lemma}\label{lem:stability1}  Assume that  $\operatorname{Lip}(h_t) \le L $ for all  $t \in [T]$. Then,  $\phi_{\le T}$ is Lipschitz continuous with 
\begin{align*}
\operatorname{Lip}( \phi_{\le T}) 
 \le \sqrt{T} (1 + L)^{t-1}.
\end{align*}
\end{lemma}
\begin{proof}
Let  $u  , u'  \in \cX^T$  and set $\delta_t = \norm{ \phi_{\le t} ( u_{\le t}) - \phi_{\le t} ( u'_{\le t}) }$.  Since $h_1$ is a constant function, $\phi_1(u_1) = h_1 + u_1$ and   $\delta_1 = \norm{ \phi_1(u_1) - \phi_1(u'_1) }  = \norm{ u_1 - u'_1 }$. For  $t = 2, \dots, T$ we can write
\begin{align*}
\delta_t  - \delta_{t-1} 
& \le  \norm{ \phi_{ t} ( u_{\le t}) - \phi_{t} ( u'_{\le t}) } \\
& =  \norm{  h_t( \phi_{ <  t} ( u_{< t}))   + u_t - h_t(  \phi_{<  t} ( u_{< t}))  - u'_t } \\
& \le    L  \delta_{t-1}  + \norm{u_t - u'_t} 
\end{align*}
Summing the increments yields 
\begin{align*}
\delta_T   & \le \sum_{t =0}^{T-1}  (1+ L)^t \norm{ u_{T - t } - u'_{T - t}}\\
& \le\left(  \sum_{t =0}^{T-1}  (1+ L)^{2t} \right)^{1/2} \norm{ u_{\le t} - u'_{\le t}}\\
& \le \sqrt{T} (1+L)^{T-1} \norm{ u_{\le t} - u'_{\le t}},
\end{align*}
where the second step is the Cauchy-Schwarz inequality. 
\end{proof}

\begin{lemma}\label{lem:stability2}
Let $x = (x_t)_{t \in [T]}$ be generated by \eqref{eq:ht} with control $u_t \equiv 0$ such that $x_{\le t} = \phi_{\le t}(0)$. Let $y = (y_t)_{[T]}$ be another sequence and assume that there are nonnegative numbers $a, b \in [0, \infty)^T$  such that 
\begin{align*}
 \norm{ h_t(y_{<t}  + u)  - h_t(y_{<t})   }  \le a_t \norm{u } + b_t , \quad \text{for all $u \in \cX^{t-1}$}
\end{align*}
Then, the sequences $x$ and $y$ satisfy 
\begin{align*}
\norm{x - y  } & \le \Big( \prod_{t=2}^{T} (1 + a_t) \Big)  \sum_{t=1}^T   (b_t  + \norm{ y_{ t} - h_{ t}(y_{<t})} )
\end{align*}
In particular, if $\Lip(h_t) \le L$ for all $t \in [T]$ then
\begin{align*}
\norm{x - y } & \le (1 + L)^{T-1}   \sum_{t=1}^T   \norm{ y_{t} - h_{ t}(y_{<t})}.
\end{align*}
\end{lemma}
\begin{proof} 
The sequence $y$ can be generated by \eqref{eq:ht} using the control term $u_t  =  y_t  - h_t(y_{<t})$ such that $y_{\le t} = \phi_{\le t}(u_{\le t})$. Define $\delta_t = \norm{ x_{\le t} - y_{\le t}}$.  Then, $\delta_1 = \norm{ x_1 - y_1} = \norm{ u_1}$ since $h_1$ is a constant function. For $t \ge 2$, we have
\begin{align*}
\delta_t - \delta_{t-1}  & \le   \norm{ x_{ t} - y_{ t}  } \\
&  =   \norm{h_t(x_{<t} )  - h_t(y_{<t}) - u_t}  \\
& \le  \norm{h_t(x_{<t} )  - h_t(y_{<t})} + \norm{u_t} \\
& \le a_t  \, \delta_{t-1}  + b_t + \norm{u_t} .
\end{align*}
Summing the increments, up to time $t$, we find that
\begin{align*}
\delta_t & \le  \sum_{s =1}^{t-1} \Big(  \prod_{r =s+1}^{t}  (1 + a_r)  \Big) ( b_s+ \norm{u_s})  +  b_t + \norm{u_t} \\
& \le  \Big(  \prod_{r =2}^{t}  (1 + a_r)  \Big)   \sum_{s =1}^{t} ( b_s+ \norm{u_s}),
\end{align*}
Evaluating at $t = T$ completes the proof. 
\end{proof}

\subsection{Proof of Lemma~\ref{lem:SE}}

The debiasing coefficients can be expressed as
\begin{align*}
b_{st} = \frac{1}{n}  \Sigma^{-1}  \ex{ W^\top f(y_{\le t} )}
\end{align*}
where $e_s$ is the $s$-th standard basis vector in $\bbR^T$. Letting $w =\gvec(W)$ we can write
\begin{align*}
e_s^\top \Sigma^{-1} W^\top = w^\top ( \Sigma^{-1} e_s \otimes \Id_n)
\end{align*}
and this leads to 
\begin{align*}
b_{st} 
& = \frac{1}{n} \ex*{w^\top ( \Sigma^{-1} e_s \otimes \Id_n) f_t((y_{<t})}\\
& = \frac{1}{n} \ex*{ \gtr\left( f_t((y_{<t})  w^\top ( \Sigma^{-1} e_s \otimes \Id_n)  \right) }\\
& = \frac{1}{n}  \gtr\left( \cov(  f_t((y_{<t}) , w  ) ( \Sigma^{-1} e_s  \otimes \Id_n) \right)
\end{align*}
Similar to the proof Theorem~\ref{thm:XY} we can express $y_t$ as a function of $w_{\le t}$ according to the mapping $\phi_{t} \colon (\bbR^n)^t \to \bbR^n$, which is weakly differentiable by the inverse mapping theorem.  By Stein's lemma, 
\begin{align*}
\cov(  f_t((y_{<t}) , w  )  & =  \cov(f_t(\phi_{<t}(w_{<t})) , w ) = \ex*{ \frac{ \dd  f_t  }{\dd  w} } \cov( w)  
\end{align*}
Noting that $\cov(w) = \Sigma \otimes \Id_n$  then leads to 
\begin{align*}
b_{st}  & = \frac{1}{n} \gtr\left(  \ex*{ \frac{ \dd  f_t  }{\dd  w} } ( e_s \otimes \Id_n) \right)   =  \frac{1}{n} \gtr \left(  \ex*{ \frac{ \dd f_t  }{\dd  w_s} } \right) 
\end{align*}
Thus, we have recovered the expression used for the standard specification of the AMP algorithm in \eqref{eq:bst_amp}. 

Next,  we show how the partial derivative can be expressed with respect to the process  $(y_t)$.   Let $y = \gvec(Y)$, $m = \gvec(M(Y))$, and $f = \gvec(F(Y))$. Differentiating both side of the equation  $y = m + w$ and applying the chain rule leads to 
\begin{align*}
\frac{\dd y}{ \dd w}  
 =  \frac{ \dd  m}{ \dd w}    + \Id \quad \iff \quad \frac{\dd y}{ \dd w}  
 =  \frac{ \dd  m}{ \dd y} \frac{ \dd  y}{ \dd w}    + \Id  \quad \iff \quad   \left( \Id -  \frac{ \dd  m}{ \dd y} \right) \frac{\dd y}{ \dd w} = \Id
 \end{align*}
The fact that $ \frac{ \dd  m}{ \dd y}$ is a strictly lower triangular block matrix ensures that $\Id -  \frac{ \dd  m}{ \dd y}$ is nonsingular, and thus 
\begin{align*}
\frac{\dd y}{ \dd w} & = \left( \Id -  \frac{ \dd  m}{ \dd y} \right)^{-1}  =  \Id +  \left( \Id -  \frac{ \dd  m}{ \dd y} \right)^{-1} \frac{ \dd  m}{ \dd y}
\end{align*}
Combining with the chain rule,  $\frac{ \dd f_t  }{\dd  w}  = \frac{ \dd f_t}{ \dd y} \frac{ \dd y}{ \dd w} $ it follows that
\begin{align*}
b_{st}  
& = \frac{1}{n}  \ex*{  \gtr\left(  \frac{\dd  f_t }{ \dd y_s}  \right)}  + \frac{1}{n}  \ex*{  \gtr\left(  \frac{\dd  f_t }{ \dd y }  \left( \Id -  \frac{ \dd  m}{ \dd y} \right)^{-1}  \frac{ \dd  m}{ \dd y_s}  \right)   } .
 \end{align*}
Finally,  the block triangular structure of $\frac{ \dd  m}{ \dd y}$  ensures that $y$ and $m$  can be replaced with $y_{\le t}$ and $m_{\le t}$, and this gives the expression in \eqref{eq:bts_stein}.

\bibliography{gfom_coupling.bbl}


\end{document}